\documentclass[10pt,twocolumn,letterpaper]{article}

\usepackage{cvpr}              

\usepackage{url}
\usepackage[font=small,skip=2pt]{caption} 
\usepackage{lineno}
\usepackage{subcaption}
\usepackage{pgffor}
\usepackage[toc,page,header]{appendix}
\usepackage{minitoc}
\usepackage{wrapfig}
\usepackage[utf8]{inputenc} 
\usepackage[T1]{fontenc}    
\usepackage{url}            
\usepackage{booktabs}       
\usepackage{amsfonts}       
\usepackage{nicefrac}       
\usepackage{microtype}      
\usepackage{xcolor}         
\usepackage{amsmath}
\usepackage{algorithm}
\usepackage{algpseudocode}
\usepackage{amssymb}
\usepackage{booktabs}
\usepackage{multirow}
\usepackage{graphicx}
\usepackage{float}
\usepackage{enumitem}
\usepackage{tcolorbox}
\usepackage{amsthm}
\theoremstyle{plain}
\theoremstyle{plain}
\usepackage{amsmath}
\usepackage{amssymb}
\usepackage{mathtools}
\usepackage{amsthm}
\newtheorem{theorem}{Theorem}

\newtheorem{lemma}[theorem]{Lemma}

\theoremstyle{definition}

\newtheorem{assumption}[theorem]{Assumption}
\newtheorem{remark}[theorem]{Remark}
\usepackage{graphicx}
\usepackage{subcaption}
\usepackage[table]{xcolor} 
\usepackage{titletoc}
\usepackage{pifont} 
\newcommand{\AppendixTOC}[1][2]{%
  \begingroup
  \setcounter{tocdepth}{#1}%
  \printcontents[app]{l}{1}{}%
  \endgroup
}

\renewcommand{\eqref}[1]{(\ref{#1})}

\definecolor{cvprblue}{rgb}{0.21,0.49,0.74}
\usepackage[pagebackref,breaklinks,colorlinks,allcolors=cvprblue]{hyperref}

\def\paperID{\#\#\#\#} 
\def\confName{CVPR}
\def\confYear{2026}

\title{Fed-ADE: Adaptive Learning Rate for Federated Post-adaptation under Distribution Shift}
\author{Heewon Park$^1$\thanks{Equal contribution.}, Mugon Joe$^2$\footnotemark[1], Miru Kim$^1$\footnotemark[1], Kyungjin Im$^{2}$, Minhae Kwon$^{1}$\thanks{Corresponding author. {\tt minhae.kwon@skku.edu}}\\
$^1$ Department of Electrical and Computer Engineering, Sungkyunkwan University, Republic of Korea\\
$^2$ Soongsil University, Seoul, Republic of Korea \\
}

\begin{document}
\maketitle
\begin{abstract}
Federated learning (FL) in post-deployment settings must adapt to non-stationary data streams across heterogeneous clients without access to ground-truth labels. A major challenge is learning rate selection under client-specific, time-varying distribution shifts, where fixed learning rates often lead to underfitting or divergence. We propose Fed-ADE (\textit{\textbf{Fed}erated \textbf{A}daptation with \textbf{D}istribution Shift \textbf{E}stimation}), an unsupervised federated adaptation framework that leverages lightweight estimators of distribution dynamics. Specifically, Fed-ADE employs uncertainty dynamics estimation to capture changes in predictive uncertainty and representation dynamics estimation to detect covariate-level feature drift, combining them into a per-client, per-timestep adaptive learning rate. We provide theoretical analyses showing that our dynamics estimation approximates the underlying distribution shift and yields dynamic regret and convergence guarantees. Experiments on image and text benchmarks under diverse distribution shifts (label and covariate) demonstrate consistent improvements over strong baselines. These results highlight that distribution shift-aware adaptation enables effective and robust federated post-adaptation under real-world non-stationarity. The code is available at~\url{https://github.com/h2w1/Fed-ADE}
\end{abstract}

\section{Introduction}
\label{sec:intro}

Machine learning models are increasingly deployed across edge devices such as smartphones, IoT sensors, and autonomous systems, where each device continuously receives a non-stationary data stream~\cite{decentralized1,decentralized2,BMIL-ASAP,BMIL-Cluvar}. FL has emerged as a key paradigm to support such decentralized learning under privacy constraints, enabling model collaboration without raw data sharing~\cite{fedavg,fedprox,BMIL-Partial,ewha-FL1,FedAlt/Sim,yang2024fedas,ewha-FL2,FedRep,ditto}. In practice, models are pre-trained on centralized datasets and deployed to client devices, but their performance often degrades quickly due to distribution shifts in the real-time data. These shifts appear in multiple forms, such as label distribution shift and covariate shift, and evolve in client-specific and heterogeneous ways~\cite{shift1,BMIL-ICT,FedPOE,BMIL-iot,choi-tt,ko-tt,BMIL-offon}.

This environment introduces two major sources of heterogeneity: (i) shift heterogeneity, where each client experiences different temporal dynamics of distribution shift~\cite{BMIL-iot,FedPOE,shift1,shift2,shift3,shift4,BMIL-JCN}, and (ii) data heterogeneity, where clients’ local datasets already differ in scale, domain, or semantics~\cite{FedRep,ditto,yang2024fedas,FedAlt/Sim,BMIL-SL}. Together, these factors make federated post-deployment adaptation highly challenging, as a single global model or fixed training learning rate is insufficient to cope with diverse and evolving conditions.

A critical yet underexplored aspect of this problem is the choice of learning rate. In online adaptation, learning rate selection directly governs both stability and responsiveness. Fixed learning rates fail to reflect data distribution shifts, leading to underfitting when the learning rate is too small, or divergence when the learning rate is too large~\cite{ATLAS,lr1,lr2,young-opt}. The difficulty is amplified in FL, where hundreds of heterogeneous clients undergo distinct, time-varying distribution shifts without access to ground-truth labels.

\begin{figure*}
\centering
\includegraphics[width=\textwidth]{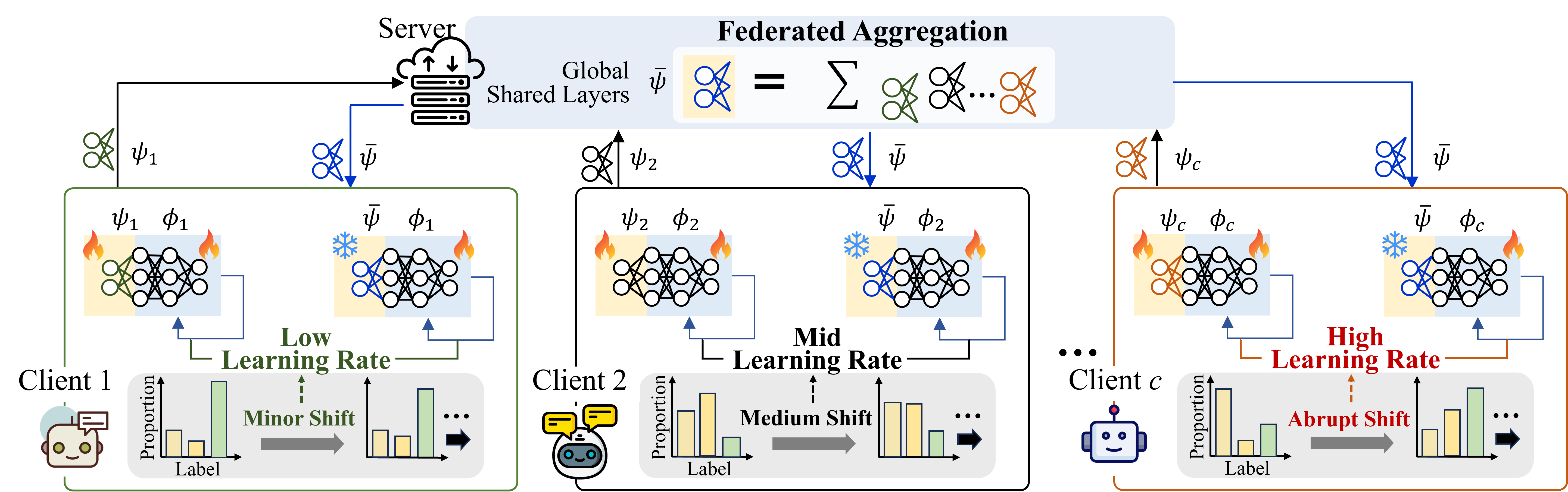}      
\caption{Overview of the Fed-ADE. Each client receives a pre-trained model, splits it into shared ($\psi_c$) and personalized ($\phi_c$) layers, and adapts to unlabeled, distribution-shifting data using an adaptive learning rate.}
\label{fig:overview}
\vspace{-0.3cm}
\end{figure*}

To address this, we propose Fed-ADE, a lightweight and unsupervised adaptation framework. Figure~\ref{fig:overview} shows the overview of the proposed Fed-ADE system. Fed-ADE introduces two estimators of distribution dynamics: \emph{Uncertainty dynamics estimation}, which captures predictive uncertainty changes, and \emph{representation dynamics estimation}, which detects embedding-level drift. By combining these signals, Fed-ADE assigns a per-client, per-timestep adaptive learning rate that reflects local distribution. Unlike prior approaches relying on multi-model ensembles or costly hyperparameter search, Fed-ADE requires no extra communication or supervision.
Furthermore, we provide theoretical analyses showing that dynamics estimation approximates underlying distribution evolution and yields dynamic regret guarantees. Empirically, we validate Fed-ADE across image and text benchmarks under label and covariate shift environments, where it consistently improves adaptation accuracy and efficiency compared to strong baselines.

In summary, our contributions are as follows.
\begin{itemize}
    \item We propose Fed-ADE, a novel framework for unsupervised federated post-adaptation under heterogeneous and time-varying distribution shifts, addressing real-time data distribution shifts.
    \item We introduce two lightweight estimators: Uncertainty dynamics estimation, which captures predictive uncertainty changes, and representation dynamics estimation, which detects embedding-level drift. These estimators enable per-client, per-timestep adaptive learning rates without requiring ground-truth labels.
    \item We provide theoretical analyses showing that distribution shift estimation approximates true distribution shift and that the resulting learning dynamics satisfy dynamic regret bound and convergence guarantees.
    \item We conduct extensive experiments on four image and a text benchmarks under label and covariate shift, demonstrating that Fed-ADE consistently outperforms state-of-the-art baselines while remaining computationally efficient.
\end{itemize}

\section{Related Works}
\label{related_work}
\paragraph{Federated Learning for Decentralized Data Environments}
FL is a widely adopted framework for collaboratively training machine learning models across multiple decentralized clients while preserving data privacy~\cite{fedavg,fedprox,yang2024fedas,FedRep,scaffold,MOON}. Early approaches such as FedAvg~\cite{fedavg} proposed aggregating locally updated model parameters on a central server without sharing raw data. However, this paradigm encounters challenges caused by data heterogeneity, as each client may hold data drawn from distinct distributions~\cite{fedprox,scaffold,MOON}. To mitigate these issues, personalized FL methods have been introduced to better adapt models to the local characteristics of each client’s data. Among them, partial-sharing strategies train only a subset of layers collaboratively via FL, while the remaining layers are updated locally~\cite{FedAlt/Sim,yang2024fedas,FedRep,BMIL-iot}. These methods enhance personalization and local adaptation, yet they commonly assume stationary or globally synchronized data distributions across clients. In contrast, our work introduces a novel system that supports post-deployment adaptation by employing an unsupervised, per-client adaptive learning rate mechanism, which dynamically adjusts to real-time distributional shifts.

\vspace{-0.2cm}
\paragraph{Distribution Shift Adaptation}
Distribution shift arises in various forms, including label and covariate shift~\cite{dist1,dist2,dist3,BMIL-OASIS,dist4,BMIL-POA}. Online label shift adaptation adjusts predictions without true labels, as in ATLAS~\cite{ATLAS} and FTH~\cite{FTH}. Covariate shift is commonly addressed by domain adaptation methods such as UDA~\cite{UDA} and UNIDA~\cite{UNIDA}, which mainly assume centralized settings and are not readily applicable to federated or online scenarios. In federated environments, adaptation approaches such as Fed-POE~\cite{FedPOE}, ATP~\cite{FL-TTA}, and FedSPL~\cite{FL-TTA2} mitigate client heterogeneity or adapt models to unlabeled target clients through ensemble strategies, adaptation-specific rules, or pseudo-labeling. However, these methods often introduce additional computational overhead or auxiliary components during adaptation.
\section{Preliminaries}
\label{problem_formulation}
\label{sec3.1}

\paragraph{Pre-training ($t=0$)}
Let $\theta$ be the global model trained at the server using pre-collected data $(\mathbf{x}_{\mathcal{G}}^0,\mathbf{y}_{\mathcal{G}}^0)$, where $\mathbf{x}_{\mathcal{G}}^0$ follows feature distribution $\mathbf{Q}_{{\mathcal{G}},\mathbf{x}}^{0}$ and $\mathbf{y}_{\mathcal{G}}^0$ follows label distribution $\mathbf{Q}_{{\mathcal{G}},\mathbf{y}}^{0}$. We also denote by $\mathbf{Q}_{{\mathcal{G}},\mathbf{x}\mid \mathbf{y}}^{0}$ the conditional distribution of features given labels, and by $\mathbf{Q}_{{\mathcal{G}},\mathbf{y}\mid \mathbf{x}}^{0}$ the conditional distribution of labels given features.
Each label $\mathbf{y}_{\mathcal{G}}^0$ is an index $i$, where $i$ is an element of the class index set $\mathcal{I}$. In the distribution $\mathbf{Q}_{\mathbf{y}_{\mathcal{G}}^0}$, the $i$-th component, $[\mathbf{Q}_{\mathbf{y}_{\mathcal{G}}^0}]_i$, represents the proportion of class $i \in \mathcal{I}$ in the label set $\mathbf{y}_{\mathcal{G}}^0$.

\vspace{-0.2cm}
\paragraph{Online Distribution Shift ($0 < t \le T$)} Consider a set of federated clients $\mathcal{C}$ collaboratively performing post-deployment adaptation. Each client model $\theta_c$ is initialized from the pre-trained global model $\theta$. At time $t$, client $c \in \mathcal{C}$ observes an unlabeled data stream $\mathbf{x}_c^t$ drawn from a client and time-dependent distribution $\mathbf{Q}_{c,\mathbf{x}}^t$, while the corresponding labels $\mathbf{y}_c^t$ follow $\mathbf{Q}_{c,\mathbf{y}}^t$ but remain unobserved. For brevity, we use $\mathbf{Q}_c^t$ to denote the overall data distribution at time $t$.  

A temporal distribution shift occurs when $\mathbf{Q}_c^t \neq \mathbf{Q}_c^{t-1}$. 
Typical cases include:  
(1) label distribution shift, where $\mathbf{Q}_{c,\mathbf{y}}^t$ changes while $\mathbf{Q}_{c,\mathbf{x}\mid\mathbf{y}}^t$ stays approximately stable and  
(2) covariate shift, where $\mathbf{Q}_{c,\mathbf{x}\mid\mathbf{y}}^t$ changes while $\mathbf{Q}_{c,\mathbf{y}}^t$ is approximately constant. Following prior works~\cite{ATLAS,FTH,sim5}, we model the shift of $\mathbf{Q}_c^t$ as
\begin{equation}
\label{eqn:label_shift}
    \mathbf{Q}_c^t = (1 - \omega(t))\,\mathbf{Q}_c^0 + \omega(t)\,\mathbf{Q}_c^T,
\end{equation}
where $\omega(t) \in [0,1]$ is a weighting function interpolating between the initial distribution $\mathbf{Q}_c^0$ and a later distribution $\mathbf{Q}_c^T$.
In the federated setting, we assume that $\mathbf{Q}_c^0$ aligns with the pre-training data distribution, whereas $\mathbf{Q}_c^T$ represents each client’s distinct distribution evolved after deployment, reflecting client-specific environments and behaviors.  

In this setting, an unsupervised federated post-adaptation framework is required to adjust to the continuously changing data distributions with limited client samples, while maintaining robust global performance.

\vspace{-0.2cm}
\paragraph{Learning Objective with Unsupervised Risk Estimation} Under the federated setting with online distribution shift, each client $c \in \mathcal{C}$ aims to update its local model $\theta_c$ using only unlabeled data. 
We adopt an unsupervised risk estimation framework that enables model updates without access to ground-truth labels. 
Let the expected risk for client $c$ at time $t$ be defined as
\begin{align} 
\label{eq:expected risk}
\mathcal{F}_c^t(\theta_c) &\triangleq \mathbb{E}_{\mathbf{y}_c^t \sim \mathbf{Q}_{c,\mathbf{y}}^t} \bigg[\mathcal{L}\Big(\mathcal{H}(\theta_c; \mathbf{x}_c^t), \mathbf{y}_c^t\Big)\bigg] \notag  \\ 
& = \sum_{i \in \mathcal{I}}[\mathbf{Q}_{c,\mathbf{y}}^t]_i \cdot \mathcal{F}_c^{t,i}(\theta_c),
\end{align}
where $\mathcal{H}(\cdot)$ denotes the softmax prediction of the model, $\mathcal{L}(\cdot)$ is the loss function, and $\mathcal{F}_c^{t,i}(\theta_c)$ is the class-wise risk for class $i$, which can be expressed as $\mathbb{E}_{\mathbf{y}_c^t \sim [\mathbf{Q}_{c,\mathbf{y}}^t]_i}\big[\mathcal{L}\big(\mathcal{H}(\theta_c; \mathbf{x}_c^t), \mathbf{y}_c^t\big)\big]$. The expected risk can be decomposed as
\begin{equation}
\label{eq:risk_decomp}
\mathcal{F}_c^t(\theta_c)
= \sum_{i \in \mathcal{I}} [\mathbf{Q}_{c,\mathbf{y}}^t]_i \, \mathcal{F}_c^{0,i}(\theta_c),
\end{equation}
where $\mathcal{F}_c^{0,i}(\theta_c)$ denotes the class-wise risk computed from the initial data $(\mathbf{x}_c^0, \mathbf{y}_c^0)$, which is assumed to follow the same distribution as the global pre-training data. 
Hence, $\mathcal{F}_c^{0,i}(\theta_c)$ can be approximated by its empirical counterpart $\widehat{\mathcal{F}}_c^{0,i}(\theta_c)$ using $(\mathbf{x}_c^0, \mathbf{y}_c^0)$.  

Since labels are not observable during post-adaptation, the current label distribution $\mathbf{Q}_{c,\mathbf{y}}^t$ must be estimated. We employ the Black-box Shift Estimation (BBSE) method~\cite{BBSE} for this purpose, which corrects the pseudo-label distribution via $\mathbf{Q}_{c,\mathbf{y}}^t \approx \mathbf{M}^{-1} \mathbf{Q}_{c,\hat{\mathbf{y}}}^t.$ Here, $\mathbf{M}$ is a confusion matrix computed on the server using pre-training data $(\mathbf{x}_{\mathcal{G}}^0, \mathbf{y}_{\mathcal{G}}^0)$ and shared with clients, while $\mathbf{Q}_{c,\hat{\mathbf{y}}}^t$ denotes the prediction distribution over the unlabeled batch $\mathbf{x}_c^t$.\footnote{Details of BBSE and the derivation of the risk estimator are provided in Appendix~\ref{appendix:proof_risk_estimation}.} Using these estimates, the unsupervised risk can be approximated as
\begin{align}
\label{eq:estimated_risk}
\widehat{\mathcal{F}}_c^t(\theta_c)
= \sum_{i \in \mathcal{I}}
\Big[\widehat{\mathbf{M}}^{-1} \mathbf{Q}_{c,\hat{\mathbf{y}}}^t\Big]_i
\cdot \widehat{\mathcal{F}}_c^{0,i}(\theta_c),
\end{align}
where $\widehat{\mathbf{M}}$ is empirical estimator of the confusion matrix.  

Finally, the federated learning objective based on the estimated unsupervised risk in~\eqref{eq:estimated_risk} is to jointly optimize all client models as follows.
\begin{align}
\label{FL_obj}
\{\theta_c^{*}\}_{c=1}^{C}
= \arg\min_{\theta_c}
\sum_{c=1}^{C} \widehat{\mathcal{F}}_c^t(\theta_c)
\end{align}

\paragraph{Personalized Federated Learning with Layer Decoupling} To achieve the goal in~\eqref{FL_obj} under heterogeneous client distributions (i.e., $\mathbf{Q}_c^t$ may differ across $c$ and shift over time), we adopt a partial-sharing FL method that decouples each client model $\theta_c$ into shared layers $\psi_c$ and personalized layers $\phi_c$, where only the shared layers $\psi_c$ are communicated with the server while the personalized layers $\phi_c$ are retained locally~\cite{FedAlt/Sim,yang2024fedas,FedRep,BMIL-iot,BMIL-SL,LG-FedAvg}.
Specifically, we adopt two step local update method~\cite{FedAlt/Sim,yang2024fedas,FedRep,BMIL-iot}. Each client participates in $R$ rounds of federated communication.
In round $r\in\{1,\ldots,R\}$, a subset of clients $\mathcal{C}^{(r)}\subset\mathcal{C}$ updates both shared and personalized layers locally with the learning rate $\eta_c^{t}$, and the risk estimator $\widehat{\mathcal{F}}_c^t\big(\{\psi_c,\phi_c\}\big) $ obtained in \eqref{eq:estimated_risk} as follows.
\begin{equation}
\label{eq:local_joint_update}
\{\psi_c,\phi_c\}\;\leftarrow\;\{\psi_c,\phi_c\}\;-\;\eta_c^t\,\nabla_{\psi_c,\phi_c}\,\widehat{\mathcal{F}}_c^t\big(\{\psi_c,\phi_c\}\big)
\end{equation}
Then, each clients sends $\psi_c$ to the server, which aggregates them using a weighted average based on each client’s local dataset size $N_{c}^{t}$ as follows.
\begin{equation}
\label{eq:aggregate}
\bar{\psi} \;=\; \frac{1}{\sum_{c\in \mathcal{C}^{(r)}} N_c^{t}} \sum_{c\in \mathcal{C}^{(r)}} N_c^{t}\,\psi_c
\end{equation}
The server broadcasts $\bar{\psi}$ to all clients, who set $\psi_c\leftarrow\bar{\psi}$ and thus hold $\theta_c=\{\bar{\psi},\phi_c\}$. After receiving $\bar{\psi}$, clients freeze the shared layers and adapt only the personalized layers as follows.
\begin{equation}
\label{eq:local_personal_update}
\phi_c\;\leftarrow\;\phi_c\;-\;\eta_c^t\,\nabla_{\phi_c}\,\widehat{\mathcal{F}}_c^t\big(\{\psi_c,\phi_c\}\big) 
\end{equation} 
With this update strategy, clients can preserve aggregated global knowledge in $\bar{\psi}$ while capturing the client-specific structure in $\phi_c$.
\begin{algorithm}[t]
\caption{\textbf{Fed-ADE}}
\label{alg:fedade}
\begin{algorithmic}[1]
\State \textbf{Inputs:} Total rounds $R$, and learning rate bounds $[\eta_{\min}, \eta_{\max}]$
\State \textbf{Initialize:} Local model $\{\psi_c,\phi_c\}$, previous summaries $(\mathbf{q}_c^{t-1},\mathbf{z}_c^{t-1})$
\For{$r = 1,2,\dots,R$} 
    \State Server samples participants $\mathcal{C}^{(r)}\subset\mathcal{C}$
    \For{$c \in \mathcal{C}^{(r)}$ \textbf{in parallel}} 
        \State Build unsupervised risk $\widehat{\mathcal{F}}_c^{t}(\psi_c,\phi_c)$ as~\eqref{eq:estimated_risk}
        \State Compute batch summaries $\mathbf{q}_c^{t}$ and $\mathbf{z}_c^{t}$
        \State Estimate dynamics $\mathcal{S}_{\mathrm{unc},c}^{t}$ as~\eqref{eq:S_lab}, $\mathcal{S}_{\mathrm{rep},c}^{t}$ as~\eqref{eq:S_rep}
        \State Combine shift dynamic signals $ \mathcal{S}_c^{t}$ as~\eqref{eq:S_combined}
        \State Set adaptive learning rate $\eta_c^{t}$ as~\eqref{eq:eta_map}
        
        \State Update local model $\{\psi_c,\phi_c\}$ by~\eqref{eq:local_joint_update}
        
        \State Send $\psi_c$ to server
    \EndFor
    \State Server aggregates to obtain $\bar{\psi}$ via~\eqref{eq:aggregate}
    \State Server broadcasts $\bar{\psi}$
    \For{$c \in \mathcal{C}^{(r)}$ \textbf{in parallel}}
        \State Replace shared layers $\psi_c\leftarrow\bar{\psi}$
        \State Update personalized layers $\phi_c$ as ~\eqref{eq:local_personal_update}
        \State Cache $(\mathbf{q}_c^{t-1},\mathbf{z}_c^{t-1})\leftarrow(\mathbf{q}_c^{t},\mathbf{z}_c^{t})$
    \EndFor
\EndFor
\end{algorithmic}
\end{algorithm}


\section{Distribution Shift Driven Federated Post-adaptation}
\label{sec:SEAD}
In this section, we introduce Fed-ADE, a federated post-adaptation framework designed to address heterogeneous and time-varying distribution shifts across clients. 
Fed-ADE adaptively controls learning rates of each client based on the degree of local distribution shift, enabling fast yet stable post-deployment adaptation. The overall process of Fed-ADE is shown in Algorithm~\ref{alg:fedade}.

\subsection{Distribution Shift Adaptive Learning Rate}
The effectiveness of adaptation in~\eqref{eq:local_joint_update} and~\eqref{eq:local_personal_update} 
hinges on selecting a learning rate $\eta_c^{t}$ that reflects how rapidly the local data distribution shifts at client $c$ at timestep $t$.  
To capture this temporal shift, Fed-ADE introduces a distribution dynamics signal $\mathcal{S}_c^{t}$ that quantifies the magnitude of the observed distribution shift.  
The learning rate is then adaptively scheduled according to this signal as follows.
\begin{equation}
\label{eq:eta_map}
\eta_{c}^{t}
= \eta_{\min} + (\eta_{\max}-\eta_{\min})\,\mathcal{S}_c^{t}
\end{equation}
The adaptive rate $\eta_c^{t}$ is bounded within $[\eta_{\min}, \eta_{\max}]$. 
A higher $\mathcal{S}_c^{t}$ indicates stronger distribution shift, resulting in a larger $\eta_c^{t}$ for faster adaptation, 
whereas a smaller $\mathcal{S}_c^{t}$ implies stability, leading to a lower $\eta_c^{t}$ for conservative updates.  
This mechanism allows Fed-ADE to automatically balance adaptivity and stability in real-time across heterogeneous clients.

\subsection{Distribution Shift Dynamics Signal}
We next describe how the dynamics signal $\mathcal{S}_c^{t}$ is constructed.  
At each timestep $t$, client $c$ estimates the extent of its local distribution shift using two complementary components:  
(i) an uncertainty-based estimator $\mathcal{S}_{\mathrm{unc}}^{t}$ that captures predictive uncertainty changes, and
(ii) a representation-based estimator $\mathcal{S}_{\mathrm{rep}}^{t}$ that measures embedding-level drift.  
These signals are combined to form a unified signal as follows.
\begin{equation}
\label{eq:S_combined}
\mathcal{S}_{c}^{t}
= \frac{1}{2}\big(\mathcal{S}_{\mathrm{unc}}^{t} + \mathcal{S}_{\mathrm{rep}}^{t}\big)\in [0,1]
\end{equation}
Here, larger $\mathcal{S}_c^{t}$ indicates stronger local drift, prompting a corresponding increase in the adaptive learning rate.

\paragraph{Uncertainty Dynamics Estimation}
To capture temporal changes in predictive uncertainty, each client summarizes its model outputs over the current data batch $\mathbf{x}_c^t$ using the mean softmax vector as
\begin{equation}
\label{eq:batch_q}
\mathbf{q}_{c}^{t}
= \frac{1}{|\mathbf{x}_c^t|}\sum_{x\in \mathbf{x}_c^t}\mathcal{H}(\{\psi_c,\phi_c\};x),
\end{equation}
where $\mathbf{q}_{c}^{t}$ represents the aggregated predictive distribution.  
This serves as an entropy-based proxy for uncertainty, providing a label-free summary of class-level belief of the model.  
Batch-level averaging normalizes for batch size and mitigates per-sample stochasticity.  
The temporal change in predictive uncertainty is then measured via the cosine distance between consecutive batches as follows.
\begin{equation}
\label{eq:S_lab}
\mathcal{S}_{\mathrm{unc}}^{t} = 1 - \operatorname{cos} \big(\mathbf{q}_{c}^{t-1},\,\mathbf{q}_{c}^{t}\big)\in [0,1]
\end{equation}
This yields values in $[0,1]$ since both $\mathbf{q}_{c}^{t-1}$ and $\mathbf{q}_{c}^{t}$ are normalized probability vectors.  
Smaller $\mathcal{S}_{\mathrm{unc},c}^{t}$ indicates stable predictions with low uncertainty drift, whereas values closer to $1$ reflect stronger predictive changes.  
Computation requires storing only $\mathbf{q}_{c}^{\,t-1}$, incurring $O(|\mathcal{I}|)$ memory overhead.

\paragraph{Representation Dynamics Estimation}
To capture shifts in the embedding space, Fed-ADE summarizes the feature representation as an $\ell_2$-normalized batch mean, which ensures that cosine distance reflects only directional changes rather than scale differences, as
\begin{equation}
\label{eq:z_mean}
\mathbf{z}_{c}^{t}
= \frac{1}{|\mathbf{x}_c^t|}\sum_{x\in \mathbf{x}_c^t}\frac{h_{\psi_c}(x)}{\|h_{\psi_c}(x)\|_2}
\in \mathbb{R}^d,
\end{equation}
where $h_{\psi_c}(\cdot)$ denotes the feature extractor parameterized by $\psi_c$, and $\mathbf{z}_{c}^{t}$ represents the batch-level latent feature vector.  
The magnitude of the feature-level shift is then computed as follows.
\begin{equation}
\label{eq:S_rep}
\mathcal{S}_{\mathrm{rep}}^{t}
= \tfrac{1}{2}\big(1 - \operatorname{cos}\!\big(\mathbf{z}_{c}^{t-1},\,\mathbf{z}_{c}^{t}\big)\big)
\in [0,1]
\end{equation}
The scaling value $\tfrac{1}{2}$ normalizes the cosine distance from its natural range $[0,2]$ to $[0,1]$, aligning it with the uncertainty dynamics signal $\mathcal{S}_{\mathrm{unc}}^{t}$.
This estimator is fully label-free, operates locally, and only requires storing $\mathbf{z}_{c}^{t-1}$, incurring $O(d)$ memory overhead.

\section{Theoretical Analyses}
\label{Theoretical Analyses}
We provide theoretical analyses underpinning Fed-ADE. We first establish that the cumulative uncertainty and representation dynamics surrogates based on cosine similarity accurately approximate their respective true distribution shifts in unlabeled federated settings. 
These cumulative quantities are then used to derive a dynamic regret bound, in which each client’s adaptive learning rate $\eta_c^{t}$ is determined by its instantaneous combined shift signal $\mathcal{S}_c^{t}$, achieving the standard min–max optimal rate under non-stationarity. The analysis that verifies the adaptive updates in Fed-ADE satisfy convergence guarantees under non-stationary environments is provided in Appendix~\ref{appendix: convergnece}.

\subsection{Cumulative Shift Surrogates and Error Bounds}
\paragraph{Cumulative Uncertainty Dynamics}
The true cumulative deviation of the predictive-distribution marginal 
$\sum_{t=1}^{T}\|\mathbf Q_{c,\mathbf y}^{t}-\mathbf Q_{c,\mathbf y}^{t-1}\|_1$ 
is unobservable in practice. 
We instead use the surrogate 
$\bar{\mathcal S}_{\mathrm{unc}}\triangleq\sum_{t=1}^{T}\mathcal S_{\mathrm{unc}}^{t}$. 
Under standard calibration, concentration, and non-degeneracy assumptions, 
$\bar{\mathcal S}_{\mathrm{unc}}$ tracks the true predictive-distribution path length up to a small additive error.

\begin{theorem}[Error Bound: Cumulative Uncertainty Dynamics]
\label{th:estimate_uncertainty}
Assume $\epsilon$-calibration with respect to $\mathbf Q_{c,\mathbf y}^{t}$, i.e., 
$\|\mathcal H(\{\psi_c,\phi_c\};\mathbf x_c^{t})-\mathbf y_c^{t}\|_2\le\epsilon_t$ in expectation. 
Then, the cumulative estimation error is bounded as
\begin{equation}\label{eq:unc_error_bound}
\Big| \bar{\mathcal S}_{\mathrm{unc}} 
- \textstyle\sum_{t=1}^{T}\|\mathbf Q_{c,\mathbf y}^{t}-\mathbf Q_{c,\mathbf y}^{t-1}\|_1 \Big|
\;\le\;
K_{\cos}\sum_{t=2}^{T}(\epsilon_t+\epsilon_{t-1}),
\end{equation}
where $K_{\cos}$ is the Lipschitz constant of the cosine similarity function.
\end{theorem}
\begin{proof}
See Appendix~\ref{appendix:shift_estimation}.
\end{proof}

\paragraph{Cumulative Representation Dynamics}
For representation-level drift, we define
$\bar{\mathcal S}_{\mathrm{rep}}\triangleq\sum_{t=1}^{T}\mathcal S_{\mathrm{rep}}^{t}$. 
Under standard concentration of normalized feature means and the Lipschitz continuity of the shared layers $h_{\psi_c}$, which extract local representations, 
$\bar{\mathcal S}_{\mathrm{rep}}$ accurately tracks the temporal evolution of local feature representations.

\begin{theorem}[Error Bound: Cumulative Representation Dynamics]
\label{th:estimate_representation}
Assume $\|\mathbf z_c^{t}-\bar{\mathbf z}_c^{t}\|_2\le\epsilon_t'$, where 
$\bar{\mathbf z}_c^{t}=\mathbb{E}[h_{\psi_c}(x)/\|h_{\psi_c}(x)\|_2]$, 
and that $h_{\psi_c}$ is $K_h$-Lipschitz. 
Then, the cumulative estimation error is bounded as
\begin{align}\label{eq:rep_error_bound}
\Big| \bar{\mathcal S}_{\mathrm{rep}} 
- \textstyle\sum_{t=1}^{T}\tfrac{1}{2}\!\left(1-\frac{\langle 
\bar{\mathbf z}_c^{t-1},\bar{\mathbf z}_c^{t}\rangle}
{\|\bar{\mathbf z}_c^{t-1}\|_2\,\|\bar{\mathbf z}_c^{t}\|_2}\right) \Big|
\notag \\
\le\;
K_h \sum_{t=2}^{T}(\epsilon_t'+\epsilon_{t-1}'),
\end{align}
where $\epsilon_t'$ denotes the feature-level approximation error.
\end{theorem}

\begin{proof} See Appendix~\ref{appendix:feature_estimation}.
\end{proof}
\vspace{-0.2cm}
\paragraph{Combined Cumulative Shift Surrogate}
We define the total cumulative shift surrogate as $\bar{\mathcal S}_c=\tfrac{1}{2}\big(\bar{\mathcal S}_{\mathrm{unc}}+\bar{\mathcal S}_{\mathrm{rep}}\big)$.
Let $\mathcal V_c$ denote the corresponding combined true shift measure. 
Then, by Theorems~\ref{th:estimate_uncertainty} and~\ref{th:estimate_representation}, and applying the triangle inequality, we obtain the following.
\begin{align*}
\label{eq:combined_error_bound}
\big|\bar{\mathcal S}_c - \mathcal V_c\big|
\;\le\;&\;
\tfrac{1}{2}\Big[ K_{\cos}\!\sum_{t=2}^{T}(\epsilon_t+\epsilon_{t-1}) 
+K_h\!\sum_{t=2}^{T}(\epsilon_t'+\epsilon_{t-1}')
\Big]
\end{align*}
Hence, the combined surrogate $\bar{\mathcal S}_c$ closely tracks the true distribution shifts and can be directly used in the subsequent dynamic regret analysis.

\subsection{Dynamic Regret Analyses}
At each timestep $t$, client $c$ updates its model parameters $\{\psi_c, \phi_c\}$ and incurs a loss 
$\mathcal{F}_c^t(\{\psi_c,\phi_c\}) = \mathcal{L}\big(\mathcal{H}(\{\psi_c,\phi_c\}, \mathbf{x}_c^t), \mathbf{y}_c^t\big)$. 
The \emph{dynamic regret} over $T$ rounds is defined as
\begin{align*}
\mathrm{Reg}_T 
\triangleq 
\sum_{t=1}^T 
\mathcal{F}_c^t(\{\psi_c,\phi_c\})
- 
\sum_{t=1}^T 
\min_{\{\psi_c^t,\phi_c^t\}} 
\mathcal{F}_c^t(\{\psi_c,\phi_c\}),
\end{align*}
where $\min_{\{\psi_c^t,\phi_c^t\}} \mathcal{F}_c^t$ denotes the instantaneous optimal loss at each round.
This formulation quantifies the additional cost incurred by non-stationarity in the local data distribution.

\paragraph{Dynamic Regret Bound} 
Prior analyses of online distribution shift~\cite{OLS1,OLS2,OLS3} provide an $\mathcal{O}(\sqrt{T})$ \emph{static regret} bound against a fixed comparator. 
In contrast, in non-stationary environments where client data distributions evolve over time, 
we establish a \emph{dynamic regret} bound that scales with the cumulative shift proxy $\bar{\mathcal{S}}_c$.

Let $G \triangleq \sup_{(\mathbf{x}_c^t, \mathbf{y}_c^t)} 
\|\nabla_{\{\psi_c, \phi_c\}}\mathcal{L}(\mathcal{H}(\{\psi_c,\phi_c\}, \mathbf{x}_c^t), \mathbf{y}_c^t)\|_2$ 
be an upper bound on the gradient norm, and 
$B \triangleq \sup_{(\mathbf{x}_c^t, \mathbf{y}_c^t)} 
|\mathcal{L}(\mathcal{H}(\{\psi_c,\phi_c\}, \mathbf{x}_c^t), \mathbf{y}_c^t)|$ 
an upper bound on the loss value. 
Under the assumptions of Lemma~\ref{lm:risk} (Appendix~\ref{appendix:proof_risk_estimation}), 
running gradient descent with a constant learning rate $\eta$ yields the following result.

\begin{theorem}[Dynamic Regret Bound]
\label{th:regretbound}
Let $\mathrm{Reg}_T$ denote the dynamic regret over $T$ rounds. 
Let $\sigma>0$ be the minimum singular value of the confusion matrix $\mathbf{M}$, 
$\Gamma$ a projection constant, and $|\mathcal{I}|$ the number of classes. 
Then, for learning rate $\eta$, the expected regret satisfies
\begin{align}
\mathbb{E}[\mathrm{Reg}_T]
&\le 
2\!\left(\frac{|\mathcal{I}| G^2}{\sigma^2} + B^2\right)\eta T 
+ \frac{\Gamma^2}{\eta} 
+ \frac{4(\Gamma+1)\sqrt{B\bar{\mathcal{S}}_c T}}{\eta} \notag\\
&\approx\mathcal{O}\!\left(\eta T + \frac{1}{\eta} + \frac{\sqrt{\bar{\mathcal{S}}_c T}}{\eta}\right).
\end{align}
\end{theorem}
\begin{proof}See Appendix~\ref{appendix:proof_regretbound}.
\end{proof}

The three terms $\eta T$, $1/\eta$, and $\sqrt{\bar{\mathcal{S}}_c T}/\eta$ 
correspond to optimization error, projection cost, and adaptation cost due to distributional shifts, respectively. 
By tuning the learning rate $\eta$ with respect to $\bar{\mathcal{S}}_c$, 
Fed-ADE achieves the min–max optimal dynamic regret rate under non-stationarity.

\vspace{-0.2cm}
\paragraph{Min–max Optimality} 
Fed-ADE employs the adaptive learning rate rule~\eqref{eq:eta_map}, parameterized by $[\eta_{\min}, \eta_{\max}]$, 
where $\mathcal{S}_c^t$ denotes the instantaneous shift signal at round $t$. 
Choosing $\eta_{\min}$ and $\eta_{\max}$ such that 
$\eta_{\min}\le\eta^*\le\eta_{\max}$, 
with $\eta^*=\Theta(T^{-1/3}\bar{\mathcal{S}}_c^{1/3})$ being the theoretically optimal rate, 
ensures that $\eta_c^t$ remains within the valid range. 
This choice satisfies the convergence conditions in Theorem~\ref{th:convergence} 
and guarantees that Theorem~\ref{th:regretbound} attains min–max optimality as follows.
\begin{equation}
\mathbb{E}[\mathrm{Reg}_T]
=\mathcal{O}\!\left(\bar{\mathcal{S}}_c^{\frac{1}{3}}T^{\frac{2}{3}}\right).
\end{equation}
Further description for min-max optimality is provided in Appendix~\ref{appendix:min-max}.

\begin{remark}
The regret rate $\mathcal{O}(\bar{\mathcal{S}}_c^{1/3}T^{2/3})$ achieved by Fed-ADE 
matches the min–max optimal bound for non-stationary online learning under unsupervised label shift~\cite{besbes}. 
The adaptive rate bound $[\eta_{\min},\eta_{\max}]$ 
is designed to contain the optimal value estimated by $\bar{\mathcal{S}}_c$, 
aligning theoretical guarantees with practical adaptivity in dynamic federated environments.
\end{remark}
\section{Simulation Results}
\label{sec:simulation}
In this section, we conduct extensive experiments to validate the effectiveness of the proposed methods. We first briefly introduce the simulation setups and then present numerical results on various datasets. The detailed description of the simulation setups is provided in Appendix~\ref{appeindix:simulation setup}, and results of extra simulations are provided in Appendix~\ref{app:extra}.

\subsection{Simulation Settings}
\label{sec:simulation_setup}
\paragraph{Shift Scenarios and Datasets}
We evaluate Fed-ADE across three prototypical online distribution-shift settings—label shift and covariate shift—using standard image benchmarks (Tiny ImageNet~\cite{Tinyimagenet}, CIFAR-10~\cite{CIFAR10}, CIFAR-100~\cite{CIFAR10}, CIFAR-10-C~\cite{CIFARC}, CIFAR-100-C~\cite{CIFARC}) and the text benchmark LAMA~\cite{LAMA}. 

\begin{itemize}
  \item \textbf{Label shift} is simulated by gradually changing each client’s class prior over time, following one of four temporal schedules (linear (Lin.), sine (Sin.), square (Squ.), Bernoulli (Ber.) on Tiny ImageNet~\cite{Tinyimagenet}, CIFAR-10~\cite{CIFAR10}, and LAMA~\cite{LAMA} benchmarks.
 \item \textbf{Covariate shift} is modeled using corruption benchmarks CIFAR10-C~\cite{CIFARC} and CIFAR100-C~\cite{CIFARC}, where clients observe progressively changing corruption severity levels following one of four temporal schedules (Lin., Sin., Squ., Ber.).    
\end{itemize}

\begin{table*}[t]
\centering
\caption{Performance comparison of different online adaptation methods (average accuracy (\%) and average wall time (sec.))}
\scriptsize
\setlength{\tabcolsep}{2pt} 
\resizebox{\textwidth}{!}{%
\begin{tabular}{@{}c c | c c c c| c c c c c c}
\toprule
\multicolumn{2}{c|}{\textbf{}} & \multicolumn{4}{c|}{\textbf{Localized Learning}} & \multicolumn{6}{c}{\textbf{Federated Learning}} \\
\cmidrule(r){3-6} \cmidrule(l){7-12}
\textbf{Dataset} & \textbf{Shift} & FTH & ATLAS & UNIDA & UDA & Fed-POE & FedCCFA &FixLR(Low) & FixLR(Mid) & FixLR(High) & \textbf{Fed-ADE} \\
\midrule
\rowcolor{gray!15}  
\multicolumn{12}{c}{\textbf{(i) Label Shift Scenarios}} \\\midrule 

\multirow{4}{*}{\shortstack{Tiny\\ImageNet}}
& Lin. &  78.2{\tiny$\pm$1.0} & 76.5{\tiny$\pm$1.3} & 83.8{\tiny$\pm$0.1} & 74.5{\tiny$\pm$0.2} & 87.1{\tiny$\pm$0.2} & 84.7{\tiny$\pm$0.8} &87.5{\tiny$\pm$1.2} & 88.2{\tiny$\pm$1.1} & 86.1{\tiny$\pm$0.4} &\textbf{89.1{\tiny$\pm$0.1}}  \\
& Sin. &  77.9{\tiny$\pm$0.8} & 76.8{\tiny$\pm$1.1} & 83.2{\tiny$\pm$0.5} & 74.5{\tiny$\pm$0.2} & 87.5{\tiny$\pm$0.4} & 84.8{\tiny$\pm$0.7} &87.3{\tiny$\pm$1.4} & 88.0{\tiny$\pm$1.1}& 87.6{\tiny$\pm$0.6} & \textbf{88.9{\tiny$\pm$0.1}} \\
& Squ. & 77.2{\tiny$\pm$0.8} & 78.5{\tiny$\pm$1.4} &  83.3{\tiny$\pm$0.1} & 74.7{\tiny$\pm$0.3} & 86.4{\tiny$\pm$0.9} & 83.0{\tiny$\pm$0.9} & 87.4{\tiny$\pm$1.2} & 88.2{\tiny$\pm$0.7} & 86.4{\tiny$\pm$0.5} &\textbf{88.9{\tiny$\pm$0.1}} \\
& Ber. & 78.2{\tiny$\pm$1.1} & 77.6{\tiny$\pm$1.1} & 82.7{\tiny$\pm$0.3} & 73.8{\tiny$\pm$0.8} & 86.5{\tiny$\pm$0.7} & 83.8{\tiny$\pm$0.9} & 86.5{\tiny$\pm$1.7} & 87.8{\tiny$\pm$1.2} & 86.3{\tiny$\pm$0.6} &\textbf{88.7{\tiny$\pm$0.1}} \\\midrule

\multirow{4}{*}{CIFAR-10}
& Lin. & 31.4{\tiny$\pm$0.8} & 36.5{\tiny$\pm$4.3} &23.0{\tiny$\pm$0.3} & 33.3{\tiny$\pm$1.4} & 71.3{\tiny$\pm$3.2} & 65.8{\tiny$\pm$0.5} &70.6{\tiny$\pm$2.0} & 70.8{\tiny$\pm$2.1} & 63.8{\tiny$\pm$1.9}& \textbf{73.8{\tiny$\pm$0.6}} \\
& Sin. & 40.3{\tiny$\pm$0.9} & 43.7{\tiny$\pm$5.1} &22.9{\tiny$\pm$0.3} & 32.0{\tiny$\pm$1.3} & 71.4{\tiny$\pm$2.6} & 65.8{\tiny$\pm$0.8} & 69.4{\tiny$\pm$1.5} & 70.5{\tiny$\pm$1.6} & 64.3{\tiny$\pm$2.2} & \textbf{73.6{\tiny$\pm$0.5}}\\
& Squ. & 31.7{\tiny$\pm$0.7} & 32.3{\tiny$\pm$5.0} &23.1{\tiny$\pm$0.1} & 28.1{\tiny$\pm$1.4} & 70.6{\tiny$\pm$1.9} &65.3{\tiny$\pm$0.3} & 72.8{\tiny$\pm$2.1} & 71.6{\tiny$\pm$2.0} & 70.6{\tiny$\pm$2.5} & \textbf{72.2{\tiny$\pm$1.6}}\\
& Ber. & 30.6{\tiny$\pm$0.9} & 32.7{\tiny$\pm$5.9} &23.0{\tiny$\pm$0.1} & 28.5{\tiny$\pm$1.5} &69.6{\tiny$\pm$1.5} & 65.4{\tiny$\pm$0.4} & 68.3{\tiny$\pm$1.7} & 71.8{\tiny$\pm$1.6} & 70.0{\tiny$\pm$2.2} & \textbf{72.9{\tiny$\pm$2.2}} \\\midrule

\multirow{4}{*}{LAMA}
& Lin. & 68.3{\tiny$\pm$1.2} & 79.5{\tiny$\pm$3.2} & 31.2{\tiny$\pm$0.8} & 72.9{\tiny$\pm$2.0} &85.4{\tiny$\pm$1.3} & 
95.6{\tiny$\pm$0.1} & 86.7{\tiny$\pm$1.2} & 95.2{\tiny$\pm$2.0} & 24.6{\tiny$\pm$3.3} &  \textbf{95.8{\tiny$\pm$0.4}}\\
& Sin. & 74.7{\tiny$\pm$3.1} & 71.8{\tiny$\pm$5.0} & 31.1{\tiny$\pm$0.6} & 70.6{\tiny$\pm$6.8} & 84.0{\tiny$\pm$1.6} &91.6{\tiny$\pm$0.9} & 88.0{\tiny$\pm$0.8} & 94.7{\tiny$\pm$2.4} & 26.9{\tiny$\pm$4.1} &\textbf{95.8{\tiny$\pm$0.6}}\\
& Squ. & 70.5{\tiny$\pm$6.5} & 79.8{\tiny$\pm$0.9} & 31.2{\tiny$\pm$0.5} & 74.4{\tiny$\pm$0.2} & 84.2{\tiny$\pm$1.0} & 92.0{\tiny$\pm$0.1} & 88.6{\tiny$\pm$0.2} & 95.4{\tiny$\pm$1.2} & 26.9{\tiny$\pm$5.3} &\textbf{96.4{\tiny$\pm$0.6}}\\
& Ber. & 76.8{\tiny$\pm$0.2} & 78.0{\tiny$\pm$5.8} & 31.1{\tiny$\pm$0.5} & 70.9{\tiny$\pm$6.5} & 84.1{\tiny$\pm$0.6} & 91.1{\tiny$\pm$0.5} & 87.9{\tiny$\pm$0.3} & 94.3{\tiny$\pm$2.4} & 20.0{\tiny$\pm$4.7} & \textbf{95.9{\tiny$\pm$0.5}} \\ \midrule

\rowcolor{gray!15}  
\multicolumn{12}{c}{\textbf{(ii) Covariate Shift Scenarios}} \\\midrule 
\multirow{4}{*}{\shortstack{CIFAR-10\\CIFAR-10-C}}
& Lin. &  23.7{\tiny$\pm$0.2} & 13.9{\tiny$\pm$0.2} & 43.3{\tiny$\pm$0.6} & 45.7{\tiny$\pm$0.8} & 44.5{\tiny$\pm$0.8} & 43.1{\tiny$\pm$0.5} &63.4{\tiny$\pm$0.2} & 63.9{\tiny$\pm$0.3} & 40.6{\tiny$\pm$2.1} & \textbf{64.4{\tiny$\pm$0.2}} \\
& Sin. &  22.9{\tiny$\pm$0.2} & 13.2{\tiny$\pm$0.2} & 44.8{\tiny$\pm$0.1} & 43.5{\tiny$\pm$0.3} & 44.7{\tiny$\pm$1.2} &43.3{\tiny$\pm$0.4} &63.5{\tiny$\pm$0.3} & 63.7{\tiny$\pm$0.3} & 38.1{\tiny$\pm$3.1} & \textbf{64.8{\tiny$\pm$0.1}} \\
& Squ. & 23.8{\tiny$\pm$0.1} & 14.1{\tiny$\pm$0.4} & 42.2{\tiny$\pm$0.1} & 43.3{\tiny$\pm$0.2} & 48.5{\tiny$\pm$1.2} & 41.6{\tiny$\pm$0.7} & 62.7{\tiny$\pm$2.1} & 64.5{\tiny$\pm$2.1} & 39.6{\tiny$\pm$2.8} &\textbf{65.4{\tiny$\pm$0.2}} \\
& Ber. &  23.6{\tiny$\pm$0.2} & 14.2{\tiny$\pm$0.3} & 42.7{\tiny$\pm$0.2} & 42.3{\tiny$\pm$0.1} & 48.7{\tiny$\pm$0.9} & 42.0{\tiny$\pm$0.5} &64.1{\tiny$\pm$2.1} & 64.4{\tiny$\pm$2.1} & 40.8{\tiny$\pm$2.2} & \textbf{65.8{\tiny$\pm$0.2}}\\\midrule

\multirow{4}{*}{\shortstack{CIFAR-100\\CIFAR-100-C}}
& Lin. &  9.2{\tiny$\pm$0.1} & 3.5{\tiny$\pm$0.1} & 34.3{\tiny$\pm$0.1} & 35.6{\tiny$\pm$0.2} & 27.3{\tiny$\pm$0.1} & 24.3{\tiny$\pm$0.3} &44.0{\tiny$\pm$0.2} & 43.4{\tiny$\pm$0.6} &41.4{\tiny$\pm$0.6} & \textbf{45.8{\tiny$\pm$0.2}} \\
& Sin. & 7.6 {\tiny$\pm$0.2} & 2.9{\tiny$\pm$0.1} & 34.1{\tiny$\pm$0.1} & 35.7{\tiny$\pm$0.1} & 27.5{\tiny$\pm$0.1} &24.5{\tiny$\pm$0.5} & 43.9{\tiny$\pm$0.1} &42.1{\tiny$\pm$0.2} &39.8{\tiny$\pm$0.3} & \textbf{46.7{\tiny$\pm$0.1}}  \\
& Squ. & 8.5{\tiny$\pm$0.1} & 3.4{\tiny$\pm$0.1} & 33.8{\tiny$\pm$0.2} & 35.3{\tiny$\pm$0.1} & 31.9{\tiny$\pm$0.1} & 25.1{\tiny$\pm$0.4} & 44.9{\tiny$\pm$0.3}& 42.8{\tiny$\pm$0.2}&34.3{\tiny$\pm$0.5} & \textbf{46.5{\tiny$\pm$0.3}}  \\
& Ber. & 8.6{\tiny$\pm$0.1} & 3.5{\tiny$\pm$0.1} & 33.0{\tiny$\pm$0.1} & 34.7{\tiny$\pm$0.2} & 31.3{\tiny$\pm$0.1} &  25.2{\tiny$\pm$0.4} & 45.3{\tiny$\pm$0.5}&43.5{\tiny$\pm$0.2} &31.1{\tiny$\pm$0.4} &\textbf{46.8{\tiny$\pm$0.1}}  \\\midrule\midrule




\multicolumn{2}{c|}{\textbf{Average Wall Time}}
& 2631.2{\tiny$\pm$71.3} & 1959.2{\tiny$\pm$22.6} & 1886.5{\tiny$\pm$3.7} & 1242.7{\tiny$\pm$7.8} & 130.7{\tiny$\pm$1.9} & 210.8{\tiny$\pm$16.6} & \textbf{109.5{\tiny$\pm$0.4}} & \textbf{109.5{\tiny$\pm$0.4}} & \textbf{109.5{\tiny$\pm$0.4}} & \textbf{109.5{\tiny$\pm$0.4}}\\
\bottomrule
\end{tabular}}
\vspace{-0.3cm}
\label{sim:performance_compare}
\end{table*}

\vspace{-0.2cm}
\paragraph{Baselines for Comparing Adaptation Methods} We compare the proposed Fed-ADE with the following baselines.
\begin{itemize}[leftmargin=*, topsep=0pt, itemsep=0pt]
    \item \textbf{Localized Learning:} We evaluate four centralized unsupervised online adaptation methods—FTH~\cite{FTH}, ATLAS~\cite{ATLAS}, UNIDA~\cite{UNIDA}, and UDA~\cite{UDA}. These methods perform adaptation locally without any client collaboration, thus lacking the ability to leverage shared knowledge across users.
    \item \textbf{Federated Learning}: We evaluate two federated adaptation methods—Fed-POE~\cite{FedPOE}, and FedCCFA~\cite{concept_neurips} that combine locally fine-tuned models with aggregated global models over time, using a fixed learning rate irrespective of shift magnitude.
    \item \textbf{Fed-ADE (fixed learning rate)}: We evaluate fixed learning rate (FixLR) variants that remove adaptivity from Fed-ADE and use constant stepsizes within the same bounds as Fed-ADE; For image benchmarks, we use $5\times 10^{-6}$ for FixLR(Low), $10^{-5}$ for FixLR(Mid), and $10^{-4}$ for FixLR(High). For LAMA, we use $10^{-4}$ for FixLR(Low), $10^{-3}$ for FixLR(Mid), and $10^{-2}$ for FixLR(High). 
    \item \textbf{Fed-ADE (ours):} Our method performs personalized federated post-adaptation by dynamically adjusting the local learning rate based on estimated distribution-shift magnitude. For all image benchmarks, we set $\eta_{\min}=5\times10^{-6}$ and $\eta_{\max}=10^{-4}$, while for the text-based LAMA benchmark we use $\eta_{\min}=10^{-3}$ and $\eta_{\max}=10^{-2}$.
\end{itemize}

\vspace{-0.2cm}
\paragraph{Evaluation Protocols} At each timestep, clients adapt using unlabeled data and are evaluated on the same data with labels, reporting accuracy averaged over clients and timesteps. We assume clients do not know the pre-training distribution and initialize with uniform priors. For simulation, We set the total number of timesteps to 100 and the total number of clients to 100.

\subsection{Simulation Results}
\paragraph{Performance of Adaptation Methods} 
Table~\ref{sim:performance_compare} summarizes the results across three prototypical shift scenarios—label and covariate—over diverse datasets. Overall, Fed-ADE achieves the highest accuracy in all settings while maintaining the lowest wall-clock time. Localized adaptation methods (FTH, ATLAS, UNIDA, UDA) exhibit limited robustness and incur substantially higher computation, whereas federated approaches such as Fed-POE and FedCCFA provide improved stability but still fall short of Fed-ADE in dynamic environments.

Under label distribution shift, Fed-ADE consistently surpasses all baselines across all temporal schedules (Lin., Sin., Squ., Ber.). On Tiny ImageNet and CIFAR-10, Fed-ADE provides steady improvements over FixLR and Fed-POE, with average gains of about $1\%$ and $4\%$, respectively. A similar pattern is observed on the text-based LAMA benchmark: despite its larger label space and noisier dynamics, Fed-ADE still maintains the highest accuracy across all temporal schedules, improving over the best federated baselines by roughly $2\%$ on average. 

In covariate shift scenarios, centralized domain adaptation methods (UDA, UNIDA) struggle in the federated setting, showing notably low robustness across all corruption severities. Fed-ADE consistently surpasses both FixLR and Fed-POE, achieving an average improvement rate of roughly $3\%$ over FixLR and more than $6\%$ over Fed-POE. 


In terms of computational cost, Fed-ADE delivers superior performance with exceptional efficiency. Its average wall time ($\approx 109$ sec) is $17$–$24\times$ faster than localized methods and roughly $2\times$ faster than FedCCFA, despite achieving consistently higher accuracy. These results indicate that Fed-ADE provides both accuracy and scalability under diverse online distribution shifts, achieving dynamic robustness without sacrificing computational efficiency—unlike existing approaches that specialize in only one shift type or incur substantial overhead.

\begin{figure}[t]
\centering

\begin{subfigure}{0.78\linewidth}
  \centering
  \includegraphics[width=\linewidth]{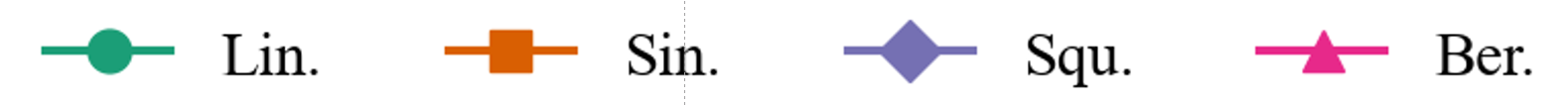} 
\end{subfigure}
\begin{subfigure}{0.48\linewidth}
  \centering
  \includegraphics[width=\linewidth]{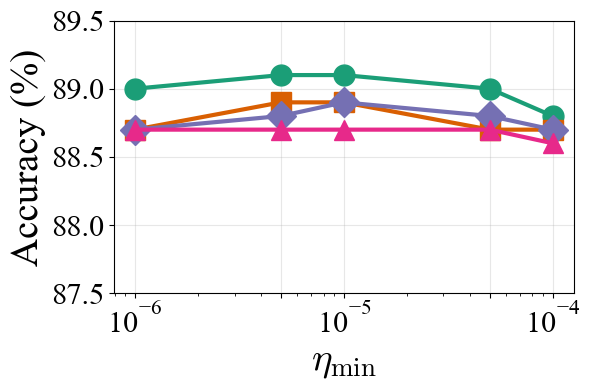}
  \caption{Tiny ImageNet $\eta_{\min}$}
\end{subfigure}\hfill
\begin{subfigure}{0.48\linewidth}
  \centering
  \includegraphics[width=\linewidth]{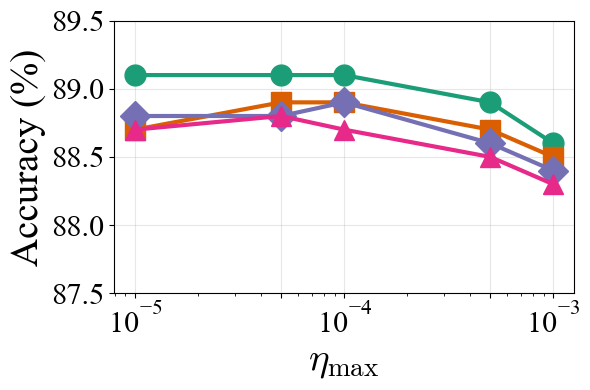}
  \caption{Tiny ImageNet $\eta_{\max}$}
\end{subfigure}
\begin{subfigure}{0.48\linewidth}
  \centering
  \includegraphics[width=\linewidth]{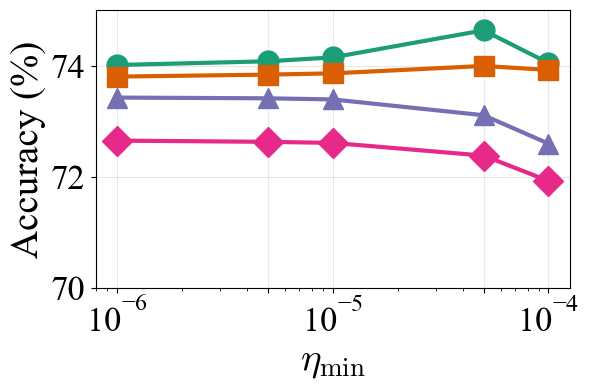}
  \caption{CIFAR-10 $\eta_{\min}$}
\end{subfigure}\hfill
\begin{subfigure}{0.48\linewidth}
  \centering
  \includegraphics[width=\linewidth]{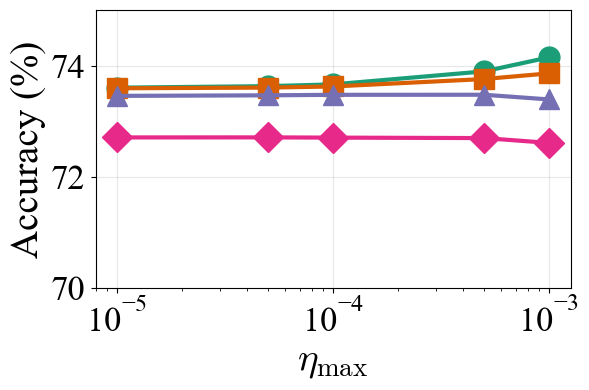}
  \caption{CIFAR-10 $\eta_{\max}$}
\end{subfigure}
\caption{Impact of learning rate bounds on Fed-ADE under label shift scenarios on Tiny-ImageNet and CIFAR-10 datasets.}
\label{fig:lr}
\vspace{-0.4cm}
\end{figure}

\begin{table}[t]
\caption{Performance with Pre-trained Model Trained on Various Distributions with CIFAR-10.}
\centering
\scriptsize
\renewcommand{\arraystretch}{1.2}
\setlength{\tabcolsep}{2pt} 
\resizebox{0.47\textwidth}{!}{%
\begin{tabular}{cc|ccccccc}
\toprule
\multicolumn{2}{c|}{\textbf{Distribution  Shift}} & \textbf{Fed-ADE} & FedCCFA & Fed-POE & UDA & UNIDA & ATLAS & FTH \\
\midrule
\multirow{4}{*}{\textbf{Gaussian}} 
& Lin. & \textbf{69.2$\pm$1.7}&62.3$\pm$1.2 & 65.7$\pm$2.1 & 20.1$\pm$2.0 & 20.2$\pm$1.8 & 25.3$\pm$2.2 & 22.3$\pm$1.7  \\
& Sin. & \textbf{68.3$\pm$1.6} &63.1$\pm$1.5  & 64.8$\pm$1.7 & 27.8$\pm$2.4 & 27.7$\pm$2.2 & 33.5$\pm$2.3 & 23.9$\pm$2.1\\
& Squ. & \textbf{66.4$\pm$1.5}& 61.9$\pm$2.6 & 61.4$\pm$2.2 & 20.3$\pm$2.1 & 19.8$\pm$1.7 & 23.1$\pm$2.0 & 23.0$\pm$1.8  \\
& Ber. & \textbf{67.1$\pm$1.3}& 63.0$\pm$2.5 & 60.9$\pm$2.5 & 18.2$\pm$2.3 & 17.7$\pm$1.9 & 22.5$\pm$2.1 & 21.6$\pm$1.9  \\
\midrule
\multirow{4}{*}{\shortstack{\textbf{Exponential}\\\textbf{Decay}}} 
& Lin. & \textbf{67.4$\pm$1.6}&54.1$\pm$3.3  & 61.1$\pm$1.8 & 23.4$\pm$3.1 & 25.4$\pm$2.8 & 26.1$\pm$1.5 & 22.3$\pm$1.9 \\
& Sin. & \textbf{67.6$\pm$1.6}& 54.2$\pm$2.5  & 61.4$\pm$2.5 & 23.9$\pm$2.9 & 24.5$\pm$2.3 & 25.6$\pm$2.1 & 21.5$\pm$1.7 \\
& Squ. & \textbf{66.9$\pm$1.5}& 52.3$\pm$3.2  & 55.3$\pm$2.7 & 19.2$\pm$4.3 & 21.7$\pm$2.0 & 26.6$\pm$2.2 & 21.4$\pm$1.7 \\
& Ber. & \textbf{66.7$\pm$1.1} & 53.0$\pm$2.8  & 57.5$\pm$2.3 & 21.8$\pm$2.8 & 19.3$\pm$3.3 & 25.3$\pm$2.7 & 22.5$\pm$2.2\\
\bottomrule
\end{tabular}
} 
\vspace{-0.3cm}
\label{tab:pretrain_dist_transposed_multirow}
\end{table}

\vspace{-0.3cm}
\paragraph{Impact of Learning Rate Selection} Figure~\ref{fig:lr} examines the sensitivity of Fed-ADE to the choice of learning rate bounds under label shift scenarios on CIFAR-10 and Tiny ImageNet. 
We vary $\eta_{\min}$ while fixing $\eta_{\max}=10^{-4}$, and vary $\eta_{\max}$ while fixing $\eta_{\min}=5\times10^{-6}$. 
Across both datasets, Fed-ADE consistently maintains high accuracy within a broad range of hyperparameter choices. 
While extreme values lead to slight performance degradation, the overall trend shows that our dynamics-driven adaptation is far less sensitive to the exact learning rate bounds compared to fixed-rate baselines. 
This robustness highlights that Fed-ADE reduces the need for extensive hyperparameter tuning, an essential property in federated and online learning where search budgets are limited.




\vspace{-0.3cm}
\paragraph{Impact of Distribution of Pre-training Data} In practical federated deployments, the exact distribution of the pre-training data on the server is often unknown to clients. To examine the robustness of Fed-ADE under such realistic conditions, we evaluate post-adaptation performance on CIFAR-10 when the pre-training data follows two non-uniform distributions: \emph{Gaussian} and \emph{Exponential Decay}. As shown in Table~\ref{tab:pretrain_dist_transposed_multirow}, Fed-ADE maintains consistently strong performance across all distribution settings and shift types, with only a minor variation compared to the uniform case (cf. Table~\ref{sim:performance_compare}). These results demonstrate that our dynamics-based adaptation effectively generalizes even when the pre-training distribution differs from the assumed prior, validating its robustness under practical deployment scenarios.

\vspace{-0.3cm}
\paragraph{Ablation Study on the Similarity Measure in Fed-ADE}
Table~\ref{tab:ablation_similarity} compares cosine similarity with KL divergence~\cite{KL1,KL2,KL3}, Wasserstein distance~\cite{wasser1,wasser2,wasser3}, and Bayesian Check Point Detection (CPD)~\cite{CPD1,CPD2,CPD3} under four dynamic label-shift schedules on CIFAR-10. As shown in Table~\ref{tab:ablation_similarity}, our cosine similarity-based Fed-ADE achieves the highest accuracy in all cases. The key reason is that cosine similarity provides a bounded and direction-based measure of drift, making it robust to noisy pseudo-labels and class imbalance that frequently arise under online shifts. In contrast, KL and Wasserstein depend on absolute probability differences and become unstable when class probabilities are near zero, while Bayesian CPD often overreacts to transient fluctuations. Thus, cosine similarity-based Fed-ADE offers the most stable shift signal, enabling Fed-ADE to adjust learning rates reliably across clients. The detailed description of this ablation study is provided in Appendix~\ref{app:cosin}.

\begin{table}[t]
\scriptsize
\caption{Ablation study on the choice of distribution-shift similarity measure in Fed-ADE under CIFAR-10 label-shift scenarios.}
\centering
\setlength{\tabcolsep}{2pt} 
\resizebox{\columnwidth}{!}{%
\begin{tabular}{lcccc}
\toprule
\textbf{Measure} & \textbf{Lin.} & \textbf{Sin.} & \textbf{Squ.} & \textbf{Ber.} \\
\midrule
Fed-ADE (Cosine similarity) & \textbf{73.8 $\pm$ 0.6} & \textbf{73.6 $\pm$ 0.5} & \textbf{72.2$\pm$ 1.6} & \textbf{72.9 $\pm$ 2.2} \\
KL divergence & 63.2 $\pm$ 3.8 & 62.4 $\pm$ 2.3 & 61.1 $\pm$ 3.4 & 61.4 $\pm$ 3.6 \\
Wasserstein & 70.8 $\pm$ 0.6 & 69.9 $\pm$ 1.1 & 68.8 $\pm$ 1.3 & 66.5 $\pm$ 2.1 \\
Bayesian CPD & 71.5 $\pm$ 0.5 & 70.8 $\pm$ 0.8 & 69.7 $\pm$ 2.4 & 69.0 $\pm$ 1.1\\
\bottomrule
\end{tabular}
}
\vspace{-0.1cm}
\label{tab:ablation_similarity}
\end{table}

\begin{table}[t]
\centering
\scriptsize
\vspace{-0.2cm}
\caption{Ablation of Fed-ADE estimators on CIFAR-10}
\label{tab:ablation_estimators}

\begin{tabular}{cc lcc}
\toprule
$S_{\text{unc}}$ & $S_{\text{rep}}$ & Variant 
& Label Shift (\%) 
& Covariate Shift (\%)  \\
\midrule
\ding{51} & \ding{51} & \textbf{Fed-ADE (full)} & $73.8\pm0.6$ & $64.4\pm0.2$ \\
\ding{55} & \ding{51} & w/o $S_{\text{unc}}$    & $71.3\pm0.3$ & $64.2\pm0.1$ \\
\ding{51} & \ding{55} & w/o $S_{\text{rep}}$    & $73.1\pm0.4$ & $64.0\pm0.3$ \\
\ding{55} & \ding{55} & Fixed LR                & $70.8\pm2.1$ & $63.9\pm0.3$ \\
\bottomrule
\end{tabular}%
\vspace{-0.3cm}
\end{table}
Table~\ref{tab:ablation_estimators} shows that both w/o $S_{\text{unc}}$ and w/o $S_{\text{rep}}$ underperform the full Fed-ADE across label and covariate shifts, while a fixed learning rate performs worst. This indicates that $S_{\text{unc}}$ and $S_{\text{rep}}$ capture complementary distribution dynamics and are both necessary for effective learning rate adaptation. 

\vspace{-0.3cm}
\paragraph{Ablation Study on Estimators in Fed-ADE}
Table~\ref{tab:ablation_estimators} shows that removing either estimator degrades performance under both label and covariate shifts. In particular, w/o $S_{\text{unc}}$ leads to a noticeable drop under label shift, while w/o $S_{\text{rep}}$ slightly reduces performance under covariate shift. The fixed learning-rate baseline performs worst and exhibits higher variance. These results indicate that $S_{\text{unc}}$ and $S_{\text{rep}}$ capture complementary aspects of distribution dynamics, and that combining both signals yields more stable and effective learning-rate adaptation.

\section{Conclusions} \label{conclusion}
We presented Fed-ADE, a novel framework for unsupervised federated post-adaptation under evolving distribution shifts. Our approach combines lightweight estimators for online distribution shifts with a theoretically grounded adaptive learning rate and a personalized federated update scheme, providing both provable dynamic regret and convergence guarantees and high practical efficiency. Beyond its methodological novelty, Fed-ADE is validated through extensive experiments on various benchmarks, covering major shift types where it consistently outperforms strong baselines. This breadth of validation highlights the robustness, versatility, and real-world readiness of our method, positioning Fed-ADE as a principled and practical foundation for advancing adaptive, shift-aware federated learning in dynamic and heterogeneous environments.

\section*{Acknowledgment}
This work was supported by the National Research Foundation of Korea (NRF) grant funded by the Korean government (MSIT) under Grant RS-2025-02214082 and RS-2023-00278812.


{
    \small
    \bibliographystyle{ieeenat_fullname}
    \bibliography{main}
}

\clearpage
\setcounter{page}{1}
\linenumbers                               
\setcounter{linenumber}{1}   

\onecolumn
\startcontents[app]
\renewcommand{\thesection}{\Alph{section}}
\setcounter{section}{0}

\clearpage
\par\noindent\rule{\textwidth}{5pt}
\begin{center}
    \Large
    \textbf{Appendix - Fed-ADE: Adaptive Learning Rate for Federated Post-adaptation}\\
    \textbf{under Distribution Shift}
\end{center}
\par\noindent\rule{\textwidth}{5pt}

\medskip
\medskip

{\Large \textbf{Contents}}
\medskip

\AppendixTOC[2]   

\newpage

\newpage

\newpage

\section{Notation Table}
\label{appendix: notation table}

\begin{table}[ht]
    \caption{Notation Table}
    \label{tab:notation}
    \scriptsize
    \centering
    \begin{tabular}{llll}
    \hline
    \textbf{Notation} & \textbf{Description} & \textbf{Notation} & \textbf{Description} \\ \hline
    $\mathbf{x}_G^0$ & Pre-training data at server & $\mathbf{y}_G^0$ & Pre-training labels at server \\
     $\mathbf{Q}_{{\mathcal{G}},\mathbf{x}}^{0}$ & Feature distribution of $\mathbf{x}_G^0$ & $\mathbf{Q}_{{\mathcal{G}},\mathbf{y}}^{0}$ & Label distribution of $\mathbf{y}_G^0$ \\
     $\mathcal{I}$ & Set of class indices & $i$ & Index for class \\
     $t$ & Timestep & $T$ & Total number of timesteps \\
     $\mathcal{C}$ & Set of FL clients & $c$ & Client index \\
     $\theta$ & Global model parameters & $\theta_c$ & Model parameters of client $c$ \\
     $\mathbf{Q}_c^t$& Overall data distribution at time $t$ & $\omega(t)$ & Weighting function controlling the distribution at time $t$\\
     $\mathcal{H}(\cdot)$ & Softmax prediction of the model& $\mathcal{L}(\cdot)$ & Loss function\\
     $\mathcal{F}_c^{t,i}(\theta_c)$ & Class-wise risk for class $i$& $\widehat{\mathcal{F}}_c^{0,i}(\theta_c)$& Empirical risk\\
     $\mathbf{M}$ & Confusion matrix computed using pre-training data & $\mathbf{Q}_{c,\hat{\mathbf{y}}}^t$&Predicted distribution over $\mathbf{x}_c^t$\\
    $R$ & Total number of communication rounds & $r$ & Communication round \\
    $\mathcal{C}^{(r)}$ & Selected client set in round $r$ & $N_c^t$ & Number of samples of client $c$ at $t$ \\
    $\psi_c$ & Shared layers of client $c$ & $\phi_c$ & Personalized layers of client $c$ \\
    $\bar{\psi}$ & Aggregated global shared layers & $\{\psi_c, \phi_c\}$ & Model split: shared/personalized \\
    $[\eta_{\min}, \eta_{\max}]$ & Min/max learning rate bounds & $\mathcal{S}_c^t$ & Distribution dynamics signal \\
    $\mathcal{S}_{\mathrm{unc},c}^{t}$ & Uncertainty dynamic signal & $\mathcal{S}_{\mathrm{rep},c}^{t}$ & Representation dynamic signal \\
    $\mathbf{q}_c^{t}$ & Aggregated predictive distribution of current data batch & $\mathbf{z}_c^{t}$ & Batch-level latent feature vector\\
    $h_{\psi_c}(\cdot)$&Feature extractor&$K_{\cos},K_{h},K_\psi, K_\phi$ & Lipschitz constant\\
    $B$ & Loss value upper bound & $\sigma$ & Min. singular value of $\mathbf{M}$ \\
    $\Gamma$ & Projection constant & $|\mathcal{I}|$ & Number of classes \\
    $\mathrm{Reg}_T$ & Dynamic regret over $T$ rounds & $\Delta \mathcal{L}^{(0)}$ & Initial loss gap for convergence \\
    \hline
    \end{tabular}
\end{table}

\section{Proof of Unbiased Risk Estimation}
\label{appendix:proof_risk_estimation}

\begin{lemma}[Unbiased Risk Estimator]
\label{lm:risk}
Given model parameters $\theta_c$ independent of the test-time data $(\mathbf{x}_c^t, \mathbf{y}_c^t)$, 
the estimator $\widehat{\mathcal{F}}_c^t$ in~\eqref{eq:estimated_risk} satisfies
\[
\mathbb{E}_{\mathbf{y}_c^t \sim \mathbf{Q}_{c,\mathbf{y}}^t}[\widehat{\mathcal{F}}_c^t(\theta_c)] = \mathcal{F}_c^t(\theta_c),
\]
provided that the confusion matrix $\mathbf{M}$, estimated at the server using sufficient pre-training data $(\mathbf{x}_G^0, \mathbf{y}_G^0)$ ensuring $\widehat{\mathbf{M}} = \mathbf{M}$, and that each client has sufficient local initial data $(\mathbf{x}_c^0, \mathbf{y}_c^0)$ such that  $\widehat{\mathcal{F}}_c^{0}(\theta_c) = \mathcal{F}_c^{0}(\theta_c)$ for all $i \in \mathcal{I}$.
\end{lemma}

\begin{proof}
The expected risk decomposes into class-specific components:
\[
\mathcal{F}_c^t(\theta_c) = \sum_{i \in \mathcal{I}} [\mathbf{Q}_{\mathbf{y}_c^t}]_i \cdot \mathcal{F}_c^{0,i}(\theta_c)
\]
where $\mathcal{I}$ denotes the class index set. Through BBSE~\cite{BBSE}, we approximate $\mathbf{Q}_{\mathbf{y}_c^t} \approx \mathbf{M}^{-1}\mathbf{Q}_{\widehat{\mathbf{y}}_c^t}$ using the confusion matrix $\mathbf{M}$ from pre-training and empirical predictions $\mathbf{Q}_{\widehat{\mathbf{y}}_c^t}$.

Expanding the estimator's expectation:
\[
\mathbb{E}[\widehat{\mathcal{F}}_c^t] = \sum_{i \in \mathcal{I}} [\widehat{\mathbf{M}}^{-1}\mathbb{E}[\mathbf{Q}_{\widehat{\mathbf{y}}_c^t}]]_i \cdot \widehat{\mathcal{F}}_c^{0,i}(\theta_c)
\]
Under condition $\mathbb{E}[\mathbf{Q}_{\widehat{\mathbf{y}}_c^t}] = \mathbf{M}\mathbf{Q}_{\mathbf{y}_c^t}$, yielding:
\[
\mathbb{E}[\widehat{\mathcal{F}}_c^t] = \sum_{i \in \mathcal{I}} [\mathbf{Q}_{\mathbf{y}_c^t}]_i \cdot \mathcal{F}_c^{0,i}(\theta_c) = \mathcal{F}_c^t(\theta_c)
\]

The estimation error $|\mathbb{E}[\widehat{\mathcal{F}}_c^t] - \mathcal{F}_c^t|$ is bounded by $\mathcal{O}(1/\sqrt{n_0})$ via concentration inequalities, where a number of pre-training data $n_0 = |(\mathbf{x}_G^0, \mathbf{y}_G^0)|$. This becomes negligible when $n_0 \geq \Omega(|\mathcal{I}|^2/\epsilon^2)$ for error tolerance $\epsilon$. Practical implementations may employ singular value thresholding for numerical stability when inverting $\mathbf{M}$.
\end{proof}

Complete derivation expanding the expectation operator:
\[
\mathbb{E}_{\mathbf{x}_c^t}[\widehat{\mathcal{F}}_c^t] = \sum_{i \in \mathcal{I}} \mathbf{M}^{-1}\mathbb{E}[\mathbf{Q}_{\widehat{\mathbf{y}}_c^t}] \cdot \mathcal{F}_c^{0,i}(\theta_c) = \sum_{i \in \mathcal{I}} [\mathbf{Q}_{\mathbf{y}_c^t}]_i \cdot \mathcal{F}_c^{0,i}(\theta_c)
\]
confirming the unbiasedness through the relationship $\mathbf{Q}_{\mathbf{y}_c^t} = \mathbf{M}^{-1}\mathbb{E}[\mathbf{Q}_{\widehat{\mathbf{y}}_c^t}]$ established via BBSE methodology.

\begin{remark}
The complexity function $\Omega(\cdot)$ in the sample complexity $n_0 \geq \Omega(|\mathcal{I}|^2/\epsilon^2)$ denotes an asymptotic lower bound, meaning the pre-training dataset size must grow at least quadratically with the number of classes to guarantee estimation accuracy. This aligns with information-theoretic limits for distribution estimation.
\end{remark}


\section{Proof of Theorem~\ref{th:estimate_uncertainty}}
\label{appendix:shift_estimation}
We restate the theorem for completeness. 
Let $\mathbf{Q}_{c,\mathbf{y}}^t$ denote the true predictive distribution marginal at timestep $t$, 
and $\mathbf{q}_c^t$ the empirical mean softmax vector defined in~\eqref{eq:batch_q}. 
Under $\epsilon$-calibration with respect to $\mathbf{Q}_{c,\mathbf{y}}^t$, 
we have $\|\mathbf{q}_c^t - \mathbf{Q}_{c,\mathbf{y}}^t\|_2 \le \epsilon_t$ in expectation. 
The cumulative surrogate of predictive dynamics is
\[
\bar{\mathcal{S}}_{\mathrm{unc}} = \sum_{t=1}^{T} \mathcal{S}_{\mathrm{unc}}^t 
= \sum_{t=1}^{T} \big( 1 - \operatorname{cos}(\mathbf{q}_c^{t-1}, \mathbf{q}_c^{t}) \big).
\]
Our goal is to show that $\bar{\mathcal{S}}_{\mathrm{unc}}$ approximates 
the cumulative true deviation $\sum_{t=1}^{T}\|\mathbf{Q}_{c,\mathbf{y}}^t - \mathbf{Q}_{c,\mathbf{y}}^{t-1}\|_1$ 
up to an additive error of order $\mathcal{O}\big(\sum_t \epsilon_t\big)$.

\paragraph{Step 1: Lipschitz continuity of cosine distance.}
The cosine similarity function $\operatorname{cos}(\cdot,\cdot)$ 
is $K_{\cos}$-Lipschitz continuous in each of its arguments on the unit sphere, i.e.,
\[
\big|\,\operatorname{cos}(\mathbf{a}, \mathbf{b}) - \operatorname{cos}(\mathbf{a}', \mathbf{b}')\,\big|
\le K_{\cos} \big(\|\mathbf{a}-\mathbf{a}'\|_2 + \|\mathbf{b}-\mathbf{b}'\|_2\big).
\]
Since both $\mathbf{q}_c^t$ and $\mathbf{Q}_{c,\mathbf{y}}^t$ are probability vectors normalized to unit norm, 
this Lipschitz property holds directly.

\paragraph{Step 2: Bounding the surrogate deviation.}
By the definition of $\mathcal{S}_{\mathrm{unc}}^t$, we have
\begin{align*}
&\big| \mathcal{S}_{\mathrm{unc}}^t 
- \big(1 - \operatorname{cos}(\mathbf{Q}_{c,\mathbf{y}}^{t-1}, \mathbf{Q}_{c,\mathbf{y}}^{t})\big) \big| \\
&\le K_{\cos}\big(\|\mathbf{q}_c^{t-1} - \mathbf{Q}_{c,\mathbf{y}}^{t-1}\|_2 
+ \|\mathbf{q}_c^{t} - \mathbf{Q}_{c,\mathbf{y}}^{t}\|_2\big)
\le K_{\cos}(\epsilon_{t-1} + \epsilon_t).
\end{align*}

\paragraph{Step 3: Summing over timesteps.}
Summing over $t = 2, \dots, T$ yields
\[
\Big|
\sum_{t=1}^{T}\mathcal{S}_{\mathrm{unc}}^t
- \sum_{t=1}^{T} \big(1 - \operatorname{cos}(\mathbf{Q}_{c,\mathbf{y}}^{t-1}, \mathbf{Q}_{c,\mathbf{y}}^{t})\big)
\Big|
\le K_{\cos}\sum_{t=2}^{T} (\epsilon_t + \epsilon_{t-1}).
\]

\paragraph{Step 4: Relating cosine distance to total variation.}
For normalized probability vectors, 
the cosine distance upper bounds the $\ell_1$ deviation:
\[
1 - \operatorname{cos}(\mathbf{Q}_{c,\mathbf{y}}^{t-1}, \mathbf{Q}_{c,\mathbf{y}}^{t})
\le \|\mathbf{Q}_{c,\mathbf{y}}^{t-1} - \mathbf{Q}_{c,\mathbf{y}}^{t}\|_1.
\]
Combining this with the bound above yields
\[
\Big| \bar{\mathcal{S}}_{\mathrm{unc}} 
- \sum_{t=1}^{T}\|\mathbf{Q}_{c,\mathbf{y}}^{t-1} - \mathbf{Q}_{c,\mathbf{y}}^{t}\|_1 \Big|
\le K_{\cos}\sum_{t=2}^{T}(\epsilon_t+\epsilon_{t-1}).
\]
\hfill $\square$

\newpage
\section{Proof of Theorem~\ref{th:estimate_representation}}
\label{appendix:feature_estimation}
Let $\mathbf{z}_c^{t}$ denote the empirical $\ell_2$-normalized batch-mean feature vector defined in~\eqref{eq:z_mean}, and 
$\bar{\mathbf{z}}_c^{t}=\mathbb{E}[h_{\psi_c}(x)/\|h_{\psi_c}(x)\|_2]$ 
be its expectation under the local distribution at time $t$. 
We assume $\|\mathbf{z}_c^{t}-\bar{\mathbf{z}}_c^{t}\|_2\le\epsilon_t'$ in expectation, and that the shared representation extractor $h_{\psi_c}$ is $K_h$-Lipschitz.

The cumulative surrogate of representation dynamics is defined as
\[
\bar{\mathcal{S}}_{\mathrm{rep}}
= \sum_{t=1}^{T} \mathcal{S}_{\mathrm{rep}}^{t}
= \tfrac{1}{2}\sum_{t=1}^{T}\!\big(1-\operatorname{cos}(\mathbf{z}_c^{t-1},\mathbf{z}_c^{t})\big).
\]
Since all feature vectors are $\ell_2$-normalized, the cosine distance 
$\tfrac{1}{2}(1-\operatorname{cos}(\mathbf{a},\mathbf{b}))$ 
is smooth on the unit sphere, and its deviation with respect to small perturbations of $\mathbf{a}$ and $\mathbf{b}$ is bounded linearly by their $\ell_2$ distance. 
Using the Lipschitz continuity of $h_{\psi_c}$, the deviation between the empirical mean $\mathbf{z}_c^{t}$ and its population counterpart $\bar{\mathbf{z}}_c^{t}$ thus induces a bounded perturbation in the cosine term.

For each timestep $t$, we have
\[
\Big|\mathcal{S}_{\mathrm{rep}}^{t}
-\tfrac{1}{2}\!\left(1-\operatorname{cos}(\bar{\mathbf{z}}_c^{t-1},\bar{\mathbf{z}}_c^{t})\right)\Big|
\le
K_h\big(\epsilon_{t-1}'+\epsilon_t'\big).
\]
Summing over all $t=2,\dots,T$ gives
\[
\Big|
\sum_{t=1}^{T}\mathcal{S}_{\mathrm{rep}}^{t}
-\sum_{t=1}^{T}\tfrac{1}{2}\!\left(1-\operatorname{cos}(\bar{\mathbf{z}}_c^{t-1},\bar{\mathbf{z}}_c^{t})\right)
\Big|
\le
K_h\sum_{t=2}^{T}(\epsilon_t'+\epsilon_{t-1}').
\]
Therefore, the cumulative surrogate $\bar{\mathcal{S}}_{\mathrm{rep}}$ satisfies
\[
\Big| \bar{\mathcal{S}}_{\mathrm{rep}}
- \sum_{t=1}^{T}\tfrac{1}{2}\!\left(1-\frac{
\langle\bar{\mathbf{z}}_c^{t-1},\bar{\mathbf{z}}_c^{t}\rangle}{
\|\bar{\mathbf{z}}_c^{t-1}\|_2\,\|\bar{\mathbf{z}}_c^{t}\|_2}
\right)\Big|
\le
K_h\sum_{t=2}^{T}(\epsilon_t'+\epsilon_{t-1}'),
\]
which shows that $\bar{\mathcal{S}}_{\mathrm{rep}}$ accurately approximates the temporal trajectory of the expected local feature representations, with error bounded by the feature-level estimation noise scaled by the Lipschitz constant $K_h$.
\hfill$\square$

\newpage
\section{Proof of Theorem~\ref{th:regretbound}}
\label{appendix:proof_regretbound}

\subsection{Regret Bound}
We present the proof of Theorem~\ref{th:regretbound} by decomposing the dynamic regret into two terms. Our approach introduces a reference sequence that evolves in a piecewise-constant manner across predefined intervals. The first term evaluates the algorithm’s performance relative to this sequence, while the second term quantifies how well the reference sequence itself approximates the optimal comparator. For the following proof, we use $\{\psi_c^t, \phi_c^t\}$ for model parameters to differentiate the model per timestep.

\begin{proof}
The regret bound can be split into two terms by introducing a reference sequence $\{\psi_c^t, \phi_c^t\}$ that only changes every $\tau$ steps. Specifically, let $\mathcal{J}_m = [(m-1)\tau + 1, m\tau]$ denote the $m$-th interval. For each interval, the comparator $\{\psi_c^t, \phi_c^t\}$ is chosen as the best fixed decision in that interval, i.e., $\{\psi_{c}^{\mathcal{J}_m}, \phi_{c}^{\mathcal{J}_m}\} = \arg\min_{\{\psi_c^t, \phi_c^t\}} \sum_{t \in \mathcal{J}_m} \mathcal{F}_c^t(\{\psi_c^t, \phi_c^t\})$ for $t \in \mathcal{J}_m$. Then,
\begin{align*}
    \mathbb{E}_{1:T}&\left[ \sum_{t=1}^T \mathcal{F}_c^t(\{\psi_c^t, \phi_c^t\}) - \sum_{t=1}^T \mathcal{F}_c^t(\{\psi_c^{*}, \phi_c^{*}\}) \right] \\
    &= \mathbb{E}_{1:T}\left[ \sum_{t=1}^T \mathcal{F}_c^t(\{\psi_c^t, \phi_c^t\}) - \sum_{m=1}^M \sum_{t \in \mathcal{J}_m} \mathcal{F}_c^t(\{\psi_{c}^{\mathcal{J}_m}, \phi_{c}^{\mathcal{J}_m}\}) \right] \tag{a} \\
    &+ \mathbb{E}_{1:T}\left[ \sum_{m=1}^M \sum_{t \in \mathcal{J}_m} \mathcal{F}_c^t(\{\psi_{c}^{\mathcal{J}_m}, \phi_{c}^{\mathcal{J}_m}\}) - \sum_{t=1}^T \mathcal{F}_c^t(\{\psi_c^{*}, \phi_c^{*}\}) \right] \tag{b}
\end{align*}
where $M = \left\lceil \frac{T}{\tau} \right\rceil \leq T/\tau + 1$ is the number of intervals. Next, we analyze term (a) and term (b) separately.

\paragraph{Analysis of term (a)} This term represents the regret of the algorithm compared to the piecewise-stationary reference sequence. The regret with respect to the expected risk $\mathcal{F}_c^t(\cdot)$ can be related to the unbiased empirical risk estimator $\widehat{\mathcal{F}}_c^t(\cdot)$:
\begin{align*}
    \text{term (a)} &= \mathbb{E}_{1:T}\left[ \sum_{t=1}^T \mathcal{F}_c^t(\{\psi_c^t, \phi_c^t\}) - \sum_{t=1}^T \mathcal{F}_c^t(\{\psi_c^t, \phi_c^t\}) \right] \\
    &\leq \mathbb{E}_{1:T}\left[ \sum_{t=1}^T \langle \nabla \mathcal{F}_c^t(\{\psi_c^t, \phi_c^t\}), (\{\psi_c^t, \phi_c^t\}) - (\{\psi_c^t, \phi_c^t\}) \rangle \right] \\
    &= \mathbb{E}_{1:T}\left[ \sum_{t=1}^T \langle \nabla \mathcal{F}_c^t(\{\psi_c^t, \phi_c^t\}) - \nabla \widehat{\mathcal{F}}_c^t(\{\psi_c^t, \phi_c^t\}), (\{\psi_c^t, \phi_c^t\}) - (\{\psi_c^t, \phi_c^t\}) \rangle \right] \\
    &\quad + \mathbb{E}_{1:T}\left[ \sum_{t=1}^T \langle \nabla \widehat{\mathcal{F}}_c^t(\{\psi_c^t, \phi_c^t\}), (\{\psi_c^t, \phi_c^t\}) - (\{\psi_c^t, \phi_c^t\}) \rangle \right], 
\end{align*}
where the first inequality uses the convexity of $\mathcal{F}_c^t(\cdot)$. The first term above is zero because $\widehat{\mathcal{F}}_c^t$ is an unbiased estimator, i.e., $\nabla \mathcal{F}_c^t(\{\psi_c^t, \phi_c^t\}) = \mathbb{E}_t[ \nabla \widehat{\mathcal{F}}_c^t(\{\psi_c^t, \phi_c^t\}) \mid 1:t-1 ]$.
For the model sequence $\{\{\psi_c^t, \phi_c^t\}\}_{t=1}^T$ generated by Fed-ADE, the following lemma holds:

\begin{lemma} \label{lemma:lrupper}
    Under the assumptions of Theorem~\ref{th:regretbound}, Fed-ADE with learning rate $\eta > 0$ in equation~\eqref{eq:local_joint_update} satisfies
\[
\sum_{t=1}^T \left\langle \nabla \widehat{\mathcal{F}}_c^t(\{\psi_c^t, \phi_c^t\}), (\{\psi_c^t, \phi_c^t\}) - (\{\psi_c^t, \phi_c^t\}) \right\rangle \leq \frac{2\eta |\mathcal{I}| G^2 T}{\sigma^2} + \frac{2\Gamma P_T + \Gamma^2}{2\eta},
\]
where $P_T = \sum_{t=2}^T \|(\{\psi_c^t, \phi_c^t\}) - (\{\psi_c^{t-1}, \phi_c^{t-1}\})\|_2$ is the total variation of the comparator sequence and $G \triangleq \sup_{(\mathbf{x}_c^t, \mathbf{y}_c^t)} \big\|\nabla_{\{\psi_c, \phi_c\}} \mathcal{L}\big(\mathcal{H}\big(\{\psi_c, \phi_c\}, \mathbf{x}_c^t\big), \mathbf{y}_c^t\big)\big\|_2$ is an upper bound on the gradient norm.
\end{lemma}
Since the comparator sequence for term (a) changes only $M-1$ times, $P_T \leq \Gamma (M-1) \leq (\Gamma T)/\tau$. Taking expectations gives
\[
\text{term (a)} \leq \frac{2\eta |\mathcal{I}| G^2 T}{\sigma^2} + \frac{2\Gamma^2 T/\tau + \Gamma^2}{2\eta}.
\]

\paragraph{Analysis of term (b)} This term accounts for the error from changing the reference sequence. Following~\cite{besbes}, we get
\begin{align*}
\text{term (b)} &\leq 2\tau \sum_{t=2}^T \sup_{\{\psi_c^t, \phi_c^t\}} |\mathcal{F}_c^t(\{\psi_c^t, \phi_c^t\}) - \mathcal{F}_c^{t-1}(\{\psi_c^t, \phi_c^t\})| \triangleq 2\tau \bar{\mathcal{S}}_c.
\end{align*}
Combining the two terms, we obtain
\begin{align*}
\mathbb{E}_{1:T}\left[\sum_{t=1}^{T} \mathcal{F}_c^t(\{\psi_c^t, \phi_c^t\})\right] - \sum_{t=1}^{T} \mathcal{F}_c^t(\{\psi_c^{*}, \phi_c^{*}\}) 
&\leq \frac{2\eta |\mathcal{I}| G^2 T}{\sigma^2} + \frac{2\Gamma^2 T/\tau + \Gamma^2}{2\eta} + 2B\tau \bar{\mathcal{S}}_c \\
&\leq \left(\frac{2|\mathcal{I}| G^2}{\sigma^2} + 2B^2\right)\eta T + \frac{\Gamma^2}{\eta} + 4(\Gamma+1)\sqrt{\frac{B T \bar{\mathcal{S}}_c}{\eta}},
\end{align*}
where we set $\tau = \left\lceil \sqrt{\Gamma^2 T / (\eta B \bar{\mathcal{S}}_c)} \right\rceil$ for optimal balance and $B \triangleq \sup_{(\mathbf{x}_c^t, \mathbf{y}_c^t)} \big|\mathcal{L}\big(\mathcal{H}\big(\{\psi_c, \phi_c\}, \mathbf{x}_c^t\big), \mathbf{y}_c^t\big)\big|$ is an upper bound on the loss value, and $G$ represents gradient norm upper bound.

Under the distribution shift assumption, the variation $\bar{\mathcal{S}}_c$ can be further bounded by the change in class and representation priors. Thus,
\[
\mathbb{E}\left[\mathrm{Reg}_T\right]
\leq \left(\frac{2|\mathcal{I}| G^2}{\sigma^2} + 2B^2\right)\eta T + \frac{\Gamma^2}{\eta} + 4(\Gamma+1)\sqrt{\frac{B T \bar{\mathcal{S}}_c}{\eta}}.
\]
\end{proof}

\subsection{Min-max Optimality}
\label{appendix:min-max}
For online convex optimization with general convex losses,~\cite{besbes} showed that the dynamic regret has a lower bound of $\Omega\left(\bar{\mathcal{S}}_c^{1/3} T^{2/3}\right)$ when only noisy feedback is observed. Here, $\bar{\mathcal{S}}_c$ measures the total variation of the loss functions. The upper bound in Theorem~\ref{th:regretbound} matches this rate, showing that Fed-ADE achieves min-max optimality up to constants.

\begin{lemma}
    Under the same assumptions as Theorem~\ref{th:regretbound}, Fed-ADE with learning rate $\eta$ satisfies
\[
\mathbb{E}[\mathrm{Reg}_T] \leq 2\left(\frac{|\mathcal{I}| G^2}{\sigma^2} + B^2\right)\eta T + \frac{\Gamma^2}{\eta} + 4(\Gamma + 1)\sqrt{\frac{B \bar{\mathcal{S}}_c T}{\eta}}
= \mathcal{O}\left(\eta T + \frac{1}{\eta} + \frac{\sqrt{\bar{\mathcal{S}}_c T}}{\eta}\right),
\]
where $\sigma > 0$ is the minimum singular value of the confusion matrix $\mathbf{M}$ and $\bar{\mathcal{S}}_c$ is the temporal variation of the loss.
\end{lemma}

This result follows by setting $\eta = \Theta\left(T^{-1/3}\bar{\mathcal{S}}_c^{1/3}\right)$, which yields $\mathcal{O}\left(\bar{\mathcal{S}}_c^{1/3} T^{2/3}\right)$ regret, matching the lower bound. Similar reasoning shows that the algorithm also achieves $\mathcal{O}\left(\max\{\bar{\mathcal{S}}_c^{1/3} T^{2/3}, \sqrt{T}\}\right)$ dynamic regret.

\newpage
\section{Convergence Analyses}
\label{appendix: convergnece}
We next show that our Fed-ADE provably converges when each per-timestep learning rate $\eta_c^t\in [\eta_{min}, \eta_{max}]$ obeys the same upper bound used in Theorem~\ref{th:convergence}. For the proof, we use three commonly used assumptions~\cite{FedAlt/Sim, bottou2018optimization} as follows.
\begin{assumption}
\label{lipschitz}
    (Lipschitz parameter) For every client $c$, the loss function $\mathcal L$ is continuously differentiable and there exist \emph{Lipschitz} constants $K_{\psi}, K_{\phi}, K_{\psi\phi}, K_{\phi\psi}$ holds that 
        \begin{itemize}
    \item $\nabla_{\psi} \mathcal {L}(\bar{\psi}^{(r)},\phi_{c}^{(r)})$ is $K_{\psi}$ \textit{Lipschitz} with respect to $\psi$ and $K_{\psi\phi}$ \textit{Lipschitz} with respect to $\phi_{c}^{(r)}$.
    \item $\nabla_{\phi} \mathcal {L}(\bar{\psi}^{(r)},\phi_{c}^{(r)})$ is $K_{\phi}$ \textit{Lipschitz} with respect to $\phi$ and $K_{\phi\psi}$ \textit{Lipschitz} with respect to $\bar {\psi}^{(r)}$.
\end{itemize}
        When the average loss function of total clients is defined as
        $
            \mathcal L({\bar{\psi}^{(r)}},{\phi}^{(r)}) = \frac{1} {|\mathcal{C}|} \sum_{c \in \mathcal{C}} \mathbb{E}\big[\mathcal{L}(\bar{\psi}^{(r)},\phi_{c}^{(r)})\big]
        $
        , it has \emph{Lipschitz} constant $K_{\psi}$ with respect to $\bar{\psi}^{(r)}$, $K_{\phi\psi}/\sqrt{|\mathcal{C}|}$ with respect to ${\phi}^{(r)}$, and $K_{\psi\phi}/|\mathcal{C}|$ with respect to any ${\phi}_{c}^{(r)}$.
        Further, there exist $\chi$ to measure the relative cross-sensibility of $\nabla_{\phi} \mathcal {L}(\bar \psi^{(r)},\phi_{c}^{(r)})$ and $\nabla_{\psi} \mathcal L(\bar \psi^{(r)},\phi_c^{(r)})$ as follows.

        \begin{equation}
            \chi = \frac {\max{\{K_{\psi\phi},K_{\phi\psi}\}}} {\sqrt{K_{\psi}K_{\phi}}}
        \end{equation}
        \label{asm:lipschitz}
\end{assumption}
    
\begin{assumption}
\label{bounded variance}
    (Bounded variance) The stochastic gradients in the client-side update algorithm are unbiased and have bounded variance as follows.
            \begin{equation}
                \mathbb{E}\big[\Tilde{\nabla}_{\phi} \mathcal L(\bar{\psi}^{(r)},{\phi}_c^{(r)})\big] = {\nabla}_{\phi} \mathcal L(\bar{\psi}^{(r)},{\phi}_c^{(r)})
            \end{equation}
            \begin{equation}
                \mathbb{E}\big[\Tilde{\nabla}_{\psi} \mathcal L(\bar{\psi}^{(r)},{\phi}_c^{(r)})\big] = {\nabla}_{\psi} \mathcal L(\bar{\psi}^{(r)},{\phi}_c^{(r)})
            \end{equation}   
        Furthermore, there exist constants $\sigma_{\phi}$ and $\sigma_{\psi}$ that meet the following inequation.
            \begin{equation}
                \mathbb{E}\big[\|\Tilde{\nabla}_{\phi} \mathcal L(\bar{\psi}^{(r)},{\phi}_c^{(r)})-{\nabla}_{\phi} \mathcal L(\bar{\psi}^{(r)},{\phi}_c^{(r)})\|^2\big] \leq \sigma_{\phi}^2 
            \end{equation} 
            \begin{equation}
                \mathbb{E}\big[\|\Tilde{\nabla}_{\psi} \mathcal L(\bar{\psi}^{(r)},{\phi}_c^{(r)})-{\nabla}_{\psi} \mathcal L(\bar{\psi}^{(r)},{\phi}_c^{(r)})\|^2\big] \leq \sigma_{\psi}^2 
            \end{equation}
        We can consider ${\nabla}_{\psi} \mathcal {L}_c(\bar{\psi}^{(r)},{\phi}_c^{(r)})$ as a stochastic partial gradient of average loss function $\mathcal{L}({\bar{\psi}^{(r)}},{\phi}^{(r)})$ with respect to $\bar{\psi}^{(r)}$.
        \label{asm:bounded}
\end{assumption}
    
\begin{assumption}
\label{partial gradient}
    (Partial gradient diversity) There exist constants $\delta \geq 0 $ and $\rho \geq 0$ such that for $\bar{\psi}^{(r)}$ and all personalized layers ${\phi}^{(r)} = [{\phi}_1^{(r)},{\phi}_2^{(r)},...]$, which meet the following equations.
            \begin{equation}
                \frac{1}{|\mathcal{C}|} \sum_{c \in \mathcal{C}} \| \nabla_{\psi} \mathcal {L}_(\bar{\psi}^{(r)},{\phi}_c^{(r)}) - \nabla_{\psi} \mathcal L(\bar{\psi}^{(r)},{\phi}^{(r)}) \|^2 \leq \delta^2 + \rho^2 \|\nabla_{\psi} \mathcal{L}(\bar{\psi}^{(r)},{\phi}^{(r)}) \|^2 
            \end{equation}
        Additionally, $\mathcal L(\bar{\psi},{\phi}^{(r)})$ is bounded below by $\hat{\mathcal L}$, where $\hat{\mathcal L}$ meets the equation $\Delta \mathcal L^{(0)} = \mathcal L(\bar{\psi}^{(0)},{\phi}^{(0)}) - \hat{\mathcal L}$ at initial communication round.
        \label{asm:partial gradient}
\end{assumption}

\begin{theorem} 
\label{th:convergence}
        The convergence of the Fed-ADE is bounded as follows, where $\Delta \mathcal L^{(0)}$ denotes the difference between a bound and initial value of loss $\mathcal{L}(\bar{\psi},\phi_c)$.\footnote{To simplify the inequality in the following assumptions and theorems, we use following shorthands:
$\phi^{(r)} = \{\phi_{c}^{(r)}|c \in \mathcal{C}^{(r)}\}$, $
        \mathcal{L}(\bar{\psi}^{(r)},\phi_{c}^{(r)})=\mathcal{L}\big(\{\bar{\psi}^{(r)}, \phi_{c}^{(r)}\};\mathbf{x}_c^t\big)$, $
         \mathcal{L}(\bar{\psi}^{(r)},\phi^{(r)})=\sum_{c \in \mathcal{C}} \frac{N_c}{N}\mathcal{L}(\bar{\psi}^{(r)}, \phi_{c}^{(r)}) $, $
        \Delta_{\psi}^{(r)}=\big\|{\nabla}_{\psi} \sum_{c \in \mathcal{C}} \frac{N_c}{N}\mathcal{L}(\bar{\psi}^{(r)}, \phi_{c}^{(r)})\big\|^2 $, $
        \Delta_{\phi}^{(r)}=\frac{1}{|\mathcal{C}|}\sum_{c \in \mathcal{C}}\big\|{\nabla}_{\phi} \mathcal{L}(\bar{\psi}^{(r)}, \phi_{c}^{(r)})\big\|^2 $.}

            \begin{align}
            \label{theorem1}
            \frac{1}{R}\sum_{r=0}^{R-1}\bigg(\frac{1}{K_{\psi}} \mathbb{E}\big[\Delta_{\psi}^{(r)}\big] +\frac{|\mathcal{C}^{(r)}|}{|\mathcal{C}|K_{\phi}}\mathbb{E}\big[\Delta_{\phi}^{(r)}\big]\bigg)
            \leq \frac{\Delta \mathcal L^{(0)}}{\eta R} + \eta \sigma^{2}_{1} + \eta^2 \sigma^{2}_{2}
            \end{align}
        
        Here, the constants $\sigma_{1}^{2}$ and $\sigma_{2}^{2}$ are defined as follows. 
        \begin{align}
            &{\sigma}_{1}^{2} = \frac{\delta^2}{K_{\psi}}\Big(1-\frac{|\mathcal{C}^{(r)}|}{|\mathcal{C}|}\Big) + \frac{\sigma^{2}_{\psi}}{K_{\psi}} + \frac{\sigma^{2}_{\phi}\big(|\mathcal{C}^{(r)}|+\chi^2(|\mathcal{C}|-|\mathcal{C}^{(r)}|)\big)}{K_{\phi}|\mathcal{C}|} \\
            &{\sigma_{2}^{2}} = \big(\frac{\sigma^{2}_{\psi}+\delta^2}{K_{\psi}}+ \frac{\sigma^{2}_{\phi}|\mathcal{C}^{(r)}|}{K_{\phi}|\mathcal{C}|}\big) \big(1-E^{-1}\big) + \frac{{\chi}^2\sigma^{2}_{\phi}}{K_{\phi}}
        \end{align}
        This convergence is bounded when the learning rates are chosen as $\eta_{\psi} = \eta/(K_{\psi}E)$ and $\eta_{\phi} = \eta/(K_{\phi}E)$, where $\eta$ satisfies the following inequality.
        \begin{align}
            \eta \leq \min \biggl\{\frac{1}{24(1+\rho^2)}, \frac{|\mathcal{C}^{(r)}|}{128\chi^2(|\mathcal{C}|-|\mathcal{C}^{(r)}|)}, \sqrt{\frac{|\mathcal{C}^{(r)}|}{\chi^2 |\mathcal{C}|}} \biggr\}
        \end{align}
\end{theorem}


\begin{proof}
The proof of Theorem~\ref{th:convergence} begins with the following equation describing the update from round $r$ to round $r+1$.
    \begin{align}
    \mathcal L\Big( {\psi}^{(r+1)},{\phi}^{(r+1)}\Big) - \mathcal L\Big(\bar{\psi}^{(r)},{\phi}^{(r)}\Big) 
    &= \mathcal L\Big(\bar{\psi}^{(r)},{\phi}^{(r+1)}\Big) - \mathcal L\Big(\bar{\psi}^{(r)},{\phi}^{(r)}\Big) \label{eq:personal} \\
    & + \mathcal L\Big( {\psi}^{(r+1)},{\phi}^{(r+1)}\Big) - \mathcal L\Big(\bar{\psi}^{(r)},{\phi}^{(r+1)}\Big) \label{eq:shared}
    \end{align}
    In the case of the update for the personalized layers in~\eqref{eq:personal}, since each client is training with their own data, it can be demonstrated similarly to the gradient update in conventional deep learning approaches. However, in the case of the update for the shared layers in~\eqref{eq:shared}, as gradient descent is performed for different client models, we aim to establish its convergence.

    The smoothness bound for the updating of the shared layers gives an equation as follows.
    \begin{align}
        \label{eq:shared_part_descent_lemma}
        \mathcal L \Big( {\psi}^{(r+1)},{\phi}^{(r+1)}\Big) - \mathcal L\Big( \bar{\psi}^{(r)},{\phi}^{(r+1)}\Big) &\leq \bigg \langle \nabla_{\psi} \mathcal L\Big( \bar{\psi}^{(r)},{\phi}^{(r+1)}\Big),{\psi}^{(r+1)}-\bar{\psi}^{(r)} \bigg\rangle \notag \\  & + \frac{K_{\psi}}{2}\|{\psi}^{(r+1)}-\bar{\psi}^{(r)}\|^2 
    \end{align}

    According to Lipschitz continuity, we can express the gradient as shown in the~\eqref{eq:shared_part_descent_lemma}, when $\langle A,B \rangle$ means $A^{T}B$. The proof proceeds by demonstrating that for all communication rounds $0<r<R$, the total sum of this gradient does not diverge but rather remains bounded, as is well known.
    However, a key distinction between the conventional deep learning approach and the personalized federated learning method is that only a subset of clients among the total clients participate in training during a single round. In these conditions, the virtual full participants $\Tilde{{\phi}}^{(r+1)}$ are used to move all the dependencies of the model update on the participants $\mathcal{C}^{(r)} \in \mathcal{C}$~\cite{FedAlt/Sim}. When applying the virtual full participants to the~\eqref{eq:shared_part_descent_lemma}, we can get an equation as follows. 
    \begin{align} \label{eq:shared_part_virtual}
        \bigg \langle \nabla_{\psi} \mathcal L\Big( \bar{\psi}^{(r)},{\phi}^{(r+1)}\Big),{\psi}^{(r+1)}-\bar{\psi}^{(r)} \bigg\rangle &=\bigg \langle \nabla_{\psi} \mathcal L\Big(\bar{\psi}^{(r)},{\Tilde{\phi}}^{(r+1)}\Big), {\psi}^{(r+1)}-{\psi}^{(r)} \bigg\rangle \notag \\ &+ \bigg \langle \nabla_{\psi} \mathcal L\Big(\bar{\psi}^{(r)},{\Tilde{\phi}}^{(r+1)}\Big) -\nabla_{\psi}  \mathcal L\Big({\psi}^{(r)},{\phi}^{(r+1)}\Big),{\psi}^{(r+1)}-\bar{\psi}^{(r)} \bigg\rangle
    \end{align}

    Applying Young's inequality and Lipschitz inequality to the~\eqref{eq:shared_part_virtual}, the~\eqref{eq:shared_part_descent_lemma} can be expressed as follows.
    \begin{align}
        \mathcal L\Big( {\psi}^{(r+1)},{\phi}^{(r+1)}\Big) - \mathcal L\Big( \bar{\psi}^{(r)},{\phi}^{(r+1)}\Big) &\leq \bigg \langle \nabla_{\psi} \mathcal L\Big( \bar{\psi}^{(r)},{\phi}^{(r+1)}\Big),{\psi}^{(r+1)}-\bar{\psi}^{(r)} \bigg\rangle \notag\\
        &+ \frac{1}{2K_{\psi}}\|\nabla_{\psi} \mathcal L\Big(\bar{\psi}^{(r)},{\Tilde{\phi}}^{(r+1)}\Big) -\nabla_{\psi}  \mathcal L\Big( \bar{\psi}^{(r)},{\phi}^{(r+1)}\Big)\|^2 \notag\\
        &+ K_{\psi}\|{\psi}^{(r+1)}-\bar{\psi}^{(r)}\|^2 
    \end{align}
    
    This allows us to take an expectation. Therefore, the final bounded form can be described as follows with Assumption~\ref{asm:partial gradient}, where the $\Tilde{\Delta}_{\psi}^{(r)}$ is the analog of ${\Delta}_{\psi}^{(r)}$ with the virtual variable $\Tilde{{\phi}}^{(r+1)}$.
    \begin{align}
         &\frac{1}{R} \sum_{r=0}^{R-1}\Bigg(\frac{\eta_{\psi} E}{8}\mathbb{E}\big[\Tilde{\Delta}_{\psi}^{(r)}\big] + \frac{\eta_{\phi} E c}{16|\mathcal{C}|}\mathbb{E}\big[\Delta_{\phi}^{(r)}\big]\Bigg)\leq \frac{\Delta \mathcal L^{(0)}}{R} + \mathcal O\big(\eta_{\psi}^2+\eta_{\phi}^2\big).
         \label{eqn:virtual}
    \end{align} 

    \paragraph{Bound of shared layers update}
    When using \eqref{eqn:virtual}, the update of the shared layers can be bounded as follows.

    \begin{align}
    &\mathcal L\Big({\psi}^{(r+1)},{\phi}^{(r+1)} \Big) -\mathcal L\Big( {\psi}^{(r)},{\phi}^{(r+1)}\Big) \notag \\
    \leq & \bigg \langle \nabla_p \mathcal L\Big( {\psi}^{(r)},\Tilde{{\phi}}^{(r+1)}\Big ), {\psi}^{(r+1)} - {\psi}^{(r)} \bigg \rangle K_{\psi}\|{\psi}^{(r+1)}-{\psi}^{(r)}\|^2 + \frac{\chi^2 K_{\phi}}{2|\mathcal{C}|} \sum_{c \in \mathcal{C}} \|\Tilde{{\phi}}^{(r+1)}_c-{\phi}^{(r+1)}_c\|^2
\end{align}
Thus, the dependence of the second line term on ${\phi}^{(r+1)}$ is successfully eliminated. To bind the third line term, expectation can be used as follows.
    \begin{align}
    &\mathbb{E}\bigg [\mathcal L\Big({\psi}^{(r+1)},{\phi}^{(r+1)} \Big) - \mathcal L\Big({\psi}^{(r)},{\phi}^{(r+1)}, \Big)\bigg ] \notag \\
    &\leq -\frac{\eta_{\psi} E}{4} \mathbb{E}\big[\Tilde{\Delta}_\psi^{(r)}\big] 
    + \frac{2 \eta_{\psi} K_{\psi}^2}{|\mathcal{C}|} \sum_{c \in \mathcal{C}} \sum_{e=0}^E \mathbb{E}\big\|\Tilde{\psi}^{(r)}_{c,e} - {\psi}^{(r)}\big\|^2 + 4 \eta_{\phi}^2 E^2 K_{\phi} \sigma_{\phi}^2 \chi^2 \bigg(1-\frac{|\mathcal{C}^{(r)}|}{|\mathcal{C}|}\bigg) \notag \\
    &\quad + \frac{K_{\psi} \eta_{\psi}^2 E^2}{|\mathcal{C}^{(r)}|} \Big(\sigma_{\psi}^2+3\delta^2\big(1-\frac{|\mathcal{C}^{(r)}|}{|\mathcal{C}|}\big)\Big)+ 8 \eta_{\psi}^2 E^2 K_{\phi} \chi^2 \bigg(1-\frac{|\mathcal{C}^{(r)}|}{|\mathcal{C}|}\bigg)\Delta_{\phi}^{(r)}
    \label{eqn:expec_bound}
\end{align}
    Note that $24K_{\psi} \eta_{\psi} E (1+\rho^2)\leq 1$ is used to simplify the coefficients of some of the terms above. The term in the third line is referred to as client shift in the literature. This term can be described using the virtual variable $\Tilde{{\phi}}^{(r+1)}$ as follows.
    \begin{align}
        &\frac{2 \eta_{\psi} K_{\psi}^2}{|\mathcal{C}|} \sum_{c \in \mathcal{C}} \sum_{e=0}^E \mathbb{E}\|\Tilde{\psi}^{(r)}_{c,e} - {\psi}^{(r)}\|^2 
        \notag \\ & \leq \frac{16\eta_{\psi}^3 K_{\psi}^2 E (E-1)}{|\mathcal{C}|} \Big(\delta^2 + \rho^2\mathbb{E}\|\nabla_{\psi} \mathcal L\big( {\psi}^{(r)},\Tilde{{\phi}}^{(r+1)}\big)\|^2\Big) + 8 \eta_{\psi}^3 K_{\psi}^2 E^2 (E-1) \sigma_{\psi}^2
        \label{eqn:client_drift}
    \end{align}
    By inputting \eqref{eqn:client_drift} into the previous inequality of expectation \eqref{eqn:expec_bound}, the inequality can be described as follows.
    \begin{align}
        &\mathbb{E}[\mathcal L({\psi}^{(r+1)},{\phi}^{(r+1)}) - \mathcal L( {\psi}^{(r)},{\phi}^{(r+1)})]\notag\\
        \leq &-\frac{\eta_{\psi} E}{8} \mathbb{E}[\Tilde{\Delta}_\psi^{(r)}] + \frac{K_{\psi} \eta_{\psi}^2 E^2}{|\mathcal{C}^{(r)}|}\Big(\sigma_{\psi}^2+2\delta^2(1-\frac{|\mathcal{C}^{(r)}|}{|\mathcal{C}|})\Big)+4 \eta_{\phi}^2 E^2 K_{\phi} \sigma_{\phi}^2 \chi^2 \Big(1-\frac{|\mathcal{C}^{(r)}|}{|\mathcal{C}|}\Big) \notag \\
        +& 8 \eta_{\phi}^2 E^2 K_{\phi} \chi^2 \Big(1-\frac{|\mathcal{C}^{(r)}|}{|\mathcal{C}|}\Big)\Delta_{\phi}^{(r)} 
        +8 \eta_{\psi}^2 K_{\psi}^3 E^2 \Big(1-\frac{|\mathcal{C}^{(r)}|}{|\mathcal{C}|}\Big)(\sigma_{\psi}^2 + 2\delta_{\psi})
    \end{align}

    \paragraph{Bound of personalized layers}
    The above analysis for updating the shared layers can be applied to updating the personalized layers. This analysis provides the following, which displays the bound of the update of the personalized layers. It can be simplified some coefficients using $128 \eta_{\phi} E K_{\phi} \chi^2 (\frac{|\mathcal{C}|}{|\mathcal{C}^{(r)}|}-1) \leq 1.$

    \begin{align}
        &\mathbb{E}\Big[\mathcal L\big( {\phi}^{(r+1)},{\psi}^{(r+1)}\big)-\mathcal L\big({\psi}^{(r)},{\phi}^{(r)}\big)\Big] \notag\\
        \leq &-\frac{\eta_{\psi} E}{8} \mathbb{E}\big[{\Tilde{\Delta}_\psi^{(r)}}\big] - \frac{\eta_{\phi} E |\mathcal{C}^{(r)}|}{16|\mathcal{C}|} \mathbb{E}\big[\Delta_{\phi}^{(r)}\big]  
        + 4\eta_{\phi}^2 K_{\phi} E^2 \sigma_{\phi}^2 \bigg(\frac{|\mathcal{C}^{(r)}|}{|\mathcal{C}|} + \chi^2\Big(1-\frac{|\mathcal{C}^{(r)}|}{|\mathcal{C}|}\Big)\bigg)  \notag\\
        + & \frac{\eta^2_{\psi} K_{\psi} E^2}{|\mathcal{C}^{(r)}|} \bigg(\sigma^2_{\psi} + 2\delta^2\Big(1-\frac{|\mathcal{C}^{(r)}|}{|\mathcal{C}|}\Big)\bigg)
        + 8 \eta^2_{\psi} K_{\psi}^2 E^2(E-1)(\sigma^2_{\psi} + 2\delta^2)
        + \frac{|\mathcal{C}^{(r)}|}{|\mathcal{C}|} 4\eta_{\phi}^3 K_{\phi}^2 E^2 (E-1)\sigma^2_{\phi} 
        \label{personal_bound}
    \end{align}
    The summation of inequality \eqref{personal_bound} for considering the communication round is given below.

    \begin{align}
        &\frac{1}{R} \sum^{R-1}_{r=0}\bigg( \frac{\eta_{\psi} E}{8} \mathbb{E} \big[\Tilde{\Delta}_\psi^{(r)}\big] + \frac{\eta_{\phi} E |\mathcal{C}^{(r)}|}{|\mathcal{C}|} \mathbb{E}\big[\Delta_{\phi}^{(r)}\big] \bigg) \notag\\
        \leq &\frac{\Delta \mathcal L^{(0)}}{R} + 4\eta_{\phi}^2 K_{\phi} E^2 \sigma_{\phi}^2 \bigg(\frac{|\mathcal{C}^{(r)}|}{|\mathcal{C}|} + \chi^2\Big(1-\frac{|\mathcal{C}^{(r)}|}{|\mathcal{C}|}\Big)\bigg) 
        + \frac{\eta_{\psi}^2 K_{\psi} E^2}{|\mathcal{C}^{(r)}|} \bigg(\sigma^2_{\psi} + 2\delta^2\Big(1-\frac{|\mathcal{C}^{(r)}|}{|\mathcal{C}|}\Big)\bigg)  \notag \\
        + & 8\eta^3_{\psi} K_{\psi}^2 E^2 (E-1) (\sigma^2_{\psi} + 2\delta^2) + \frac{|\mathcal{C}^{(r)}|}{|\mathcal{C}|}4\eta_{\phi}^3 K_{\phi}^2 E^2(E-1)\sigma_{\phi}^2
    \end{align}
    This is the bound about virtual variable $\Tilde{\phi}^{(r+1)}$. To describe this bound with the participants ${\phi}^{(r+1)}$, the assumption of the Lipschitz parameter is used as follows.

    \begin{align}
    &\mathbb{E} \big\| \nabla_{\psi} \mathcal L({\psi}^{(r)},{\phi}^{(r)}) - \nabla_{\psi} \mathcal L( {\psi}^{(r)},\Tilde{{\phi}}^{(r+1)}) \big\|^2 \notag \\ &\leq \frac{1}{|\mathcal{C}|} \sum_{c \in \mathcal{C}} \mathbb{E} \big\| \nabla_{\psi} \mathcal L( {\psi}^{(r)},{\phi}_c^{(r)}) - \nabla_{\psi} \mathcal L({\psi}^{(r)},\tilde{\phi}_c^{(r+1)}) \big\|^2 \notag \\
     &\leq \frac{\chi^2 K_{\phi} K_{\psi}}{|\mathcal{C}|} \sum_{c \in \mathcal{C}} \mathbb{E} \big\| \tilde{\phi}_c^{(r+1)} - {\phi}_c^{(r)} \big\|^2 \notag \\
   &\leq \frac{\chi^2 K_{\phi} K_{\psi}}{|\mathcal{C}|} \sum_{c \in \mathcal{C}}
    \Big(16\eta_{\phi}^2 E^2 \big\| \nabla_p \mathcal L({\psi}^{(r)},{\phi}_c^{(r)}) \big\|^2 + 8\eta_{\phi}^2 E^2 \sigma^2_{\phi} \Big) = 8 \eta_{\phi}^2 E^2 \chi^2 K_{\phi} K_{\psi} \Big(\sigma_{\phi}^2 + 2\mathbb{E}\big[\Delta_{\phi}^{(r)}\big]\Big)
    \end{align}

    \noindent Using the Cauchy–Schwarz inequality described as follows.
    \begin{align}
        \label{eqn:cauchy}
        &\|\nabla_{\psi} \mathcal L({\psi}^{(r)},{\phi}^{(r)})\|^2\leq 2\|\nabla_{\psi} \mathcal L({\psi}^{(r)},{\phi}^{(r)})- \nabla_{\psi} \mathcal L({\psi}^{(r)},\Tilde{{\phi}}^{(r+1)})\|^2+ 2\|\nabla_{\psi} \mathcal L({\psi}^{(r)},\Tilde{{\phi}}^{(r+1)})\|^2
    \end{align}
    By utilizing inequality \eqref{eqn:cauchy}, we obtain the following result.
    \begin{align}
        \mathbb{E}\big[\Delta_{\phi}^{r}\big]
        \leq \mathbb{E}[\Tilde{\Delta}_\psi^{(r)}] + 16  \eta^2_{\phi} E^2 \chi^2 K_{\phi} K_{\psi} (\sigma_{\phi}^2 + 2\mathbb{E}[{\Delta}_{\phi}^{(r)}])
        \label{eqn:expec_shared}
    \end{align}
    Through inequality \eqref{eqn:expec_shared}, we can get the following inequality.
    \begin{align}
        &\frac{\eta_{\psi} E} {16} \mathbb{E}[\Delta_\psi^{(r)}] + \frac{\eta_{\phi} E |\mathcal{C}^{(r)}|}{32|\mathcal{C}|} \mathbb{E}[\Delta_{\phi}^{(r)}] 
        \leq \frac{\eta_{\psi} E}{8}\mathbb{E}[\Tilde{\Delta}_\psi^{(r)}]+ \frac{\eta_{\phi} E |\mathcal{C}^{(r)}|}{16 |\mathcal{C}|}\mathbb{E}[\Delta_{\phi}^{(r)}] 
        + \eta_{\psi} \eta_{\phi}^2 E^3 \sigma^2_{\phi} \chi^2 K_{\phi} K_{\psi}
        \label{eqn:final}
    \end{align}
    By summing inequality \eqref{eqn:final} over the total number of communication rounds and applying the parameter settings $\eta_{\phi} = \eta/(K_{\phi} E)$ and $\eta_{\psi} = \eta/(K_{\psi} E)$, the proof is completed.
\end{proof}

To ensure that the convergence analysis above remains valid when using an adaptive learning rate in \eqref{eq:eta_map}, we require that the value of the adaptive learning rate $\eta_c$ across all clients and rounds is bounded by the same upper bound used in Theorem~\ref {th:convergence}. Since our proposed label shift adaptive learning rate can not have value outside of the bound $[\eta_{min}, \eta_{max}]$ (where $\eta_c^t \in [\eta_{min}, \eta_{max}]$), by appropriately choosing $[\eta_{min}, \eta_{max}] \leq \min \biggl\{\frac{1}{24(1+\rho^2)}, \frac{|\mathcal{C}^{(r)}|}{128\chi^2(|\mathcal{C}|-|\mathcal{C}^{(r)}|)}, \sqrt{\frac{|\mathcal{C}^{(r)}|}{\chi^2 |\mathcal{C}|}} \biggr\}$, we guarantee that all learning rates used in the adaptive scheme satisfy the conditions required for Theorem~\ref{th:convergence}.

\newpage
\section{Simulation Settings}
\label{appeindix:simulation setup}
\subsection{Datasets} 

We evaluate on various benchmarks including five image-based benchmarks- Tiny ImageNet~\cite{Tinyimagenet}, CIFAR-10~\cite{CIFAR10}, CIFAR-10-C~\cite{CIFARC}, CIFAR-100~\cite{CIFAR10}, and CIFAR-100-C~\cite{CIFARC}, and a text-based benchmark LAMA~\cite{LAMA}. The following is a detailed description of the datasets used in this work.
\begin{itemize}
    \item \textbf{Tiny ImageNet~\cite{Tinyimagenet}:} Tiny ImageNet is a subset of the ImageNet dataset, containing 200 object classes, each with 500 training images, 50 validation images, and 50 test images. All images are RGB and resized to $64 \times 64$ pixels.
    \item \textbf{CIFAR-10~\cite{CIFAR10}:} The CIFAR-10 dataset consists of 60,000 RGB images of size $32 \times 32$ across 10 classes, with 50,000 images for training and 10,000 for testing. We applied standard preprocessing including normalization with dataset-wide mean and standard deviation. 
    \item \textbf{CIFAR-10-C~\cite{CIFARC}:} CIFAR-10-C is an extension of the CIFAR-10 dataset designed to evaluate model robustness against common corruptions. It contains 950,000 images created by applying 19 different types of corruption—such as noise, blur, weather effects, and digital distortions—at 5 severity levels to the original CIFAR-10 test set of 10,000 images.
    \item \textbf{CIFAR-100~\cite{CIFAR10}:} CIFAR-100 is a fine-grained variant of CIFAR-10 containing 100 object classes. It consists of 60,000 color images at 32×32 pixels, split into 50,000 training and 10,000 test images. The 100 classes are grouped into 20 superclasses for hierarchical classification research.
    \item \textbf{CIFAR-100-C~\cite{CIFARC}:} CIFAR-100C applies the same 19 corruption types with 5 severity levels from CIFAR-10-C onto the CIFAR-100 dataset. It consists of 950,000 corrupted images used as a benchmark for analyzing model robustness on fine-grained image classification tasks. 
    \item \textbf{LAMA~\cite{LAMA}:} LAnguage Model Analysis is a benchmark dataset designed to evaluate the factual and commonsense knowledge acquired by language models. In our experiments, we focus on the T-REx subset of LAMA, which consists of (subject, relation, object) triples derived from Wikidata. This subset consists of 41 relational tasks, with 100 samples randomly sampled per task. To simulate heterogeneous client environments, we assign disjoint subsets of tasks to different clients.

\end{itemize}

\subsection{Online Distribution Shift Modeling}
To emulate online distribution shift in \eqref{eqn:label_shift}, each client’s data distribution moves from a common initial uniform vector $\mathbf{Q}_c^0$  toward a client-specific target  $\mathbf{Q}_c^T$  according to a time-dependent weight $\omega(t)$.  We evaluate under four prototypical distribution shift schedules—linear (Lin.), sine (Sin.), square (Squ.), and Bernoulli (Ber.), each driving a client’s distribution from a uniform start toward a client‐specific target~\cite{ols_data1,ols_data2,ols_data3,ols_data4,ols_data5}. 
Following is a detailed description of four prototypical label-shift schedules according to different $\omega(t)$ functions, and Figure~\ref{fig:shift} shows how the parameter $\omega(t)$ changes over time for various shift types.
\begin{itemize}
    \item {Linear shift (Lin.)}: $\omega(t) = \frac{t}{T}$, where $t$ is the current timestep and $T$ is the total number of timesteps. This provides a smooth and gradual transition from $\mathbf{Q}_c^0$ to $\mathbf{Q}_c^T$ over time.
    \item {Sine shift (Sin.)}: $\omega(t) = \sin\left(\frac{\pi t}{\sqrt{T}}\right)$. This introduces cyclical variations, capturing periodic shifts in the data distribution.
    \item {Square Shift (Squ.)}: $\omega(t)$ alternates between $0$ and $1$ every $\frac{\sqrt{T}}{2}$ timestep, resulting in abrupt, periodic changes in the class distribution.
    \item {Bernoulli shift (Bern.)}: $\omega(t)$ retains its previous value $\omega(t-1)$ with a probability of $\frac{1}{\sqrt{T}}$, or flips to $1 -\omega(t-1)$. This configuration captures stochastic variations, ensuring that class priors change randomly while maintaining a scale proportional to $\sqrt{T}$.
\end{itemize}

For $\mathbf{Q}_c^T$ , we used the Dirichlet distribution to consider heterogeneous data distributions among clients. The distribution factor $\alpha$ of the Dirichlet distribution controls the level of heterogeneity among clients. Smaller $\alpha$ leads more heterogeneous distribution, and Larger $\alpha$ leads to Independent and identically distributed data system. In this study, we set $Dir(\alpha=0.1)$. Figure~\ref{fig:dirichlet} visualizes the heterogeneity among clients according to Dirichlet distribution factor $\alpha$.

\subsection{Distribution Shift Scenarios}
In all experiments, each target distribution of client $\mathbf{Q}_{\mathbf{y}_c^T}$ in~\eqref{eqn:label_shift} is defined according to the specific type of distribution shift being simulated. We summarize below the construction and datasets used for each shift scenario.
\begin{itemize}[leftmargin=*, topsep=0pt, itemsep=1pt]
    \item \textbf{Label Shift.} Label shift simulations are conducted on Tiny ImageNet~\cite{Tinyimagenet}, CIFAR-10~\cite{CIFAR10}, and LAMA~\cite{LAMA}, where each client’s class prior evolves over time according to one of four temporal schedules (Lin., Sin, Squ., Ber.) which follows~\eqref{eqn:label_shift}. 
    \item \textbf{Covariate Shift.} The target feature distribution $\mathbf{Q}_{c,\mathbf{x}}^T$ reflects corruption-induced changes to the input domain. Each client is assigned a distinct and fixed corruption type (e.g., Gaussian noise, blur), while the corruption severity level shifts dynamically over time according to one of four temporal schedules. We use the CIFAR-10~\cite{CIFAR10} $\rightarrow$ CIFAR-10-C~\cite{CIFARC} and CIFAR-100~\cite{CIFAR10} $\rightarrow$ CIFAR-100-C~\cite{CIFARC} benchmarks to simulate these covariate-shift environments.
\end{itemize}

\begin{figure}[t]
\centering
\includegraphics[width=\textwidth]{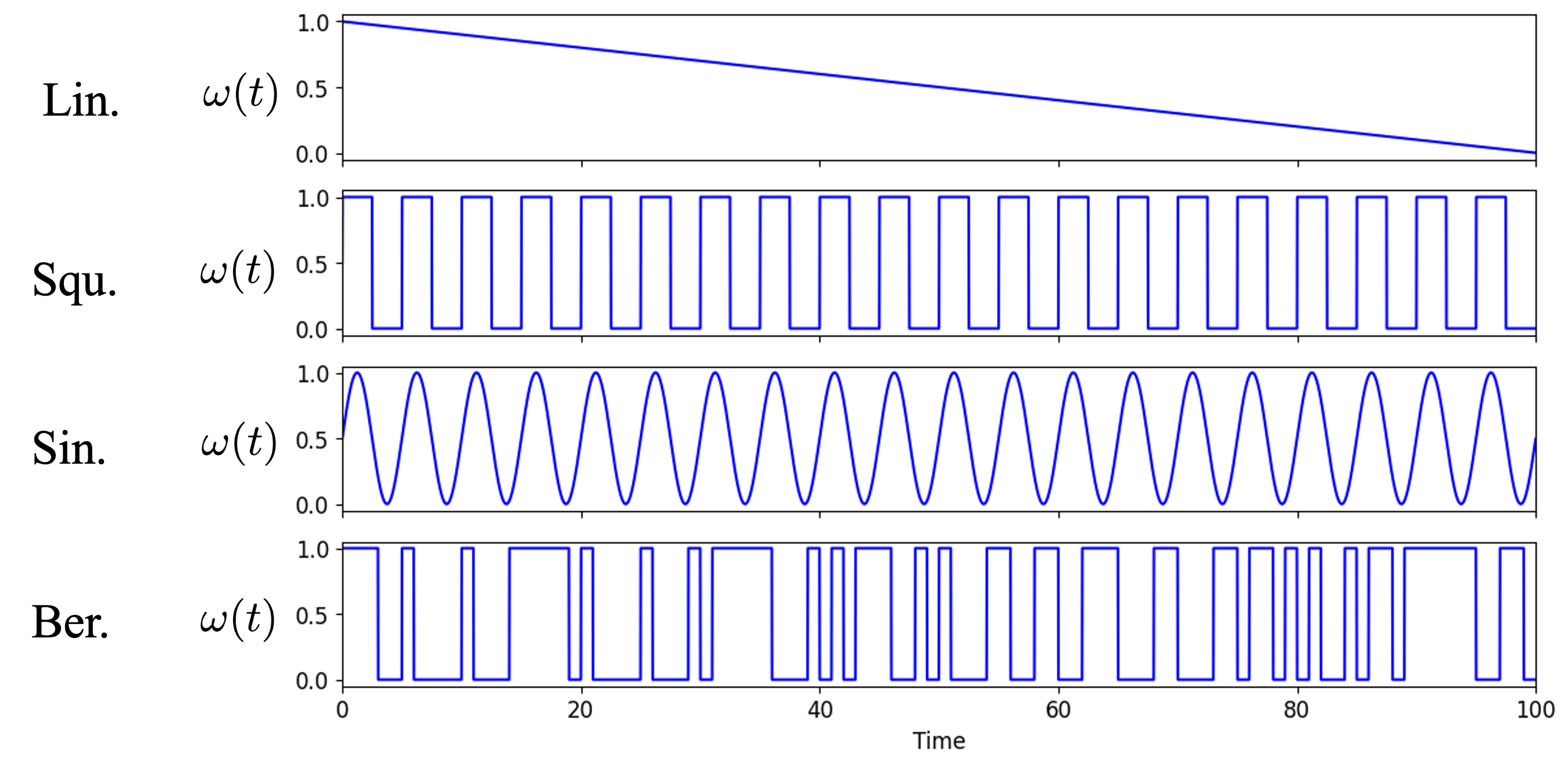} 
\hfill     
\caption{
Visualization of how the parameter $\omega(t)$ changes over time for various shift types. The Lin. presents a steady and continuous increase or decrease, reflecting a gradual and predictable change. Both the Squ. and Sin. display periodic patterns. The Squ. alternates sharply between two values at regular intervals, while the Sin. oscillates smoothly in a wave-like manner. The Ber. introduces randomness at each timestep, resulting in a stochastic and less predictable trajectory for $\omega(t)$.}
\label{fig:shift}
\end{figure}


\begin{figure}[t]

    \centering
    \begin{subfigure}[b]{0.44\textwidth}
        \includegraphics[width=\linewidth]{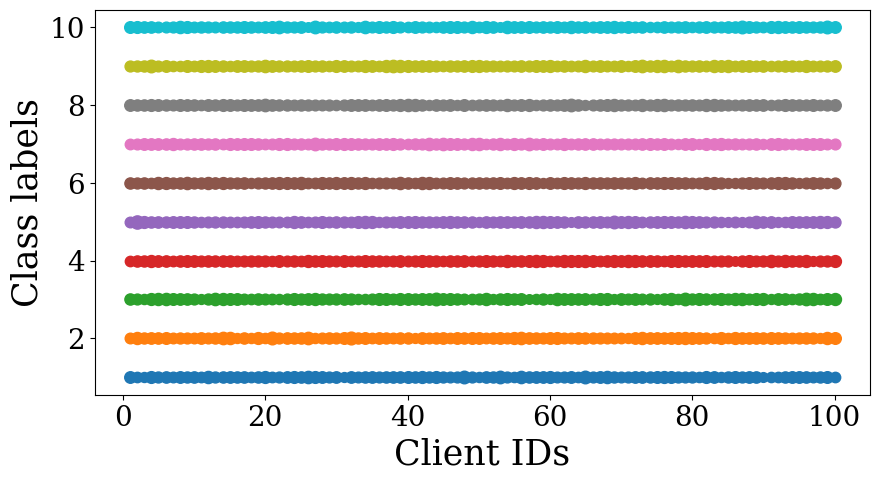}
        \caption{Dir($\alpha=100$)}
    \end{subfigure}
    \begin{subfigure}[b]{0.44\textwidth}
        \includegraphics[width=\linewidth]{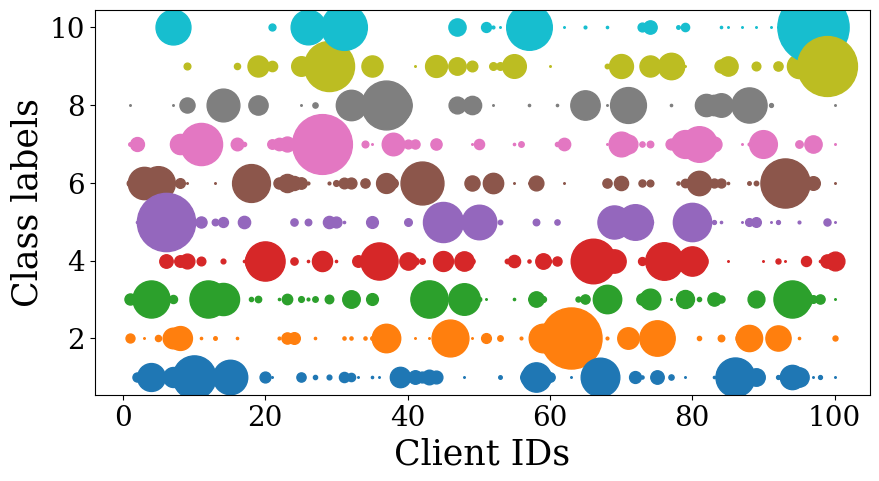}
        \caption{Dir($\alpha=0.1$)}
        
    \end{subfigure}
    
    \caption{Visualization of heterogeneity among clients according to Dirichlet distribution factor $\alpha$ in CIFAR-10 dataset. The $x$-axis represents the ID of the client, while the $y$-axis represents the class of the data. The color of the circles varies with each class type, and the size of the circles indicates the size of the data. This level of heterogeneity is not dependent on the data set but on $\alpha$ values.}
\label{fig:dirichlet}
\end{figure}

\subsection{Compared Methods}
We evaluate our proposed Fed-ADE against six existing unsupervised online adaptation baselines:

\begin{itemize}


\item \textbf{FTH}~\cite{FTH}: FTH maintains an average of estimated label distributions across all previous timesteps. By aggregating the entire history of predictions, FTH assumes that past distributional information is informative for current adaptation, making it robust to gradual, monotonic shifts but potentially less responsive to abrupt changes.


\item \textbf{ATLAS}~\cite{ATLAS}: ATLAS maintains an ensemble of base learners, each optimized with a distinct learning rate. Using an exponential weighting scheme, ATLAS dynamically aggregates these learners according to their recent performance, enabling robust adaptation to various patterns and magnitudes of label shift.

\item \textbf{UDA}~\cite{UDA}: UDA is a technique where a model is trained on labeled data from a source domain and adapted to work well on a target domain that has no labels, even though the data distributions are different. The main idea is to reduce the gap between the source and target domains by aligning their feature representations or adjusting the model so it can generalize to the target data, all without using any target domain labels.

\item \textbf{UNIDA}~\cite{UNIDA}: UNiDA is a method designed to handle domain adaptation scenarios. This assumes the same class set across domains, which aims to adapt models by distinguishing between common, source-private, and target-private classes, often using prototype-based or confidence-based techniques to avoid misclassifying unknown target classes as known ones.

\item \textbf{Fed-POE}~\cite{FedPOE}: Fed-POE is a federated online adaptation framework in which each client combines local fine-tuning with an ensemble of periodically aggregated global models. While this approach supports client personalization and leverages federated knowledge, it employs a fixed learning rate and incurs computational overhead from managing multiple model instances.

\item \textbf{FedCCFA}~\cite{concept_neurips}: FedCCFA is a federated learning framework for concept shift, where clients may experience different distributions. It performs class-level classifier clustering to share aggregated class-wise classifiers among similar clients, and aligns feature spaces using clustered feature anchors with an entropy-based adaptive weighting.

\end{itemize}

\subsection{Hyperparameters}
In this section, we summarize hyperparameters used in the simulations.

\begin{table}[h]
\centering
\caption{Hyperparameters for Tiny ImageNet, CIFAR-10, and CIFAR-100.}
\label{tab:hyperparam-a}
\scriptsize
\begin{tabular}{@{}lccc@{}}
\toprule
\textbf{Hyperparameter} & \textbf{Tiny ImageNet} & \textbf{CIFAR-10} & \textbf{CIFAR-100} \\
\midrule
Number of clients            & 100 & 100 & 100 \\
Communication rounds ($R$)  & 10  & 10  & 10 \\
Local epochs                 & 4   & 4   & 4 \\
Optimizer                    & SGD & SGD & SGD \\
Batch size                   & 128 & 32  & 32 \\
Participant rate             & 10\% & 10\% & 10\% \\
Model architecture           & ResNet18 & CNN w/ 3 residual blocks & CNN w/ 3 residual blocks \\
Shared layers ($\psi_c$)     & Representation layers & Representation layers & Representation layers \\
Lowest learning rate         & $5 \times 10^{-6}$ & $5 \times 10^{-6}$ & $5 \times 10^{-6}$ \\
Highest learning rate        & $10^{-4}$ & $10^{-4}$ & $10^{-4}$ \\
\bottomrule
\end{tabular}
\end{table}

\begin{table}[h]
\centering
\caption{Hyperparameters for CIFAR-10-C, CIFAR-100-C, and LAMA}
\label{tab:hyperparam-b}
\scriptsize
\begin{tabular}{@{}lccc@{}}
\toprule
\textbf{Hyperparameter} & \textbf{CIFAR-10-C} & \textbf{CIFAR-100-C} & \textbf{LAMA} \\
\midrule
Number of clients            & 100 & 100 & 5 \\
Communication rounds ($R$)  & 10  & 10  & 5 \\
Local epochs                 & 4   & 4   & 1 \\
Optimizer                    & SGD & SGD & AdamW \\
Batch size                   & 32  & 32  & 4 \\
Participant rate             & 10\% & 10\% & 100\% \\
Model architecture           & CNN w/ 3 residual blocks & CNN w/ 3 residual blocks & Llama-3.2-3B \\
Shared layers ($\psi_c$)     & Representation layers & Representation layers & LoRA on q/k/v (layers 0--13) \\
Lowest learning rate         & $5 \times 10^{-6}$ & $5 \times 10^{-6}$ & $1 \times 10^{-3}$ \\
Highest learning rate        & $1 \times 10^{-4}$ & $1 \times 10^{-4}$ & $1 \times 10^{-2}$ \\
\bottomrule
\end{tabular}
\end{table}

\subsection{Hardware Specifications}

\begin{table}[H]
\small
\centering
\caption{Hardware specifications used for all simulations.}
\label{tab:HW}
\begin{tabular}{@{}ll@{}}
\toprule
\textbf{Component} & \textbf{Specification} \\
\midrule
CPU    & AMD Ryzen 9 7950X, 16-Core, 32 Threads \\
GPU    & NVIDIA GeForce RTX 3090X2 \\
RAM    & 256 GB DDR5 \\
Storage & 1.8 TB NVMe SSD \\
\bottomrule
\end{tabular}
\end{table}

\newpage
\section{Extra Simulation Results}
\label{app:extra}

\subsection{Impact of Learning Rate Selection }

\begin{figure*}[t]
\centering

\begin{subfigure}{0.6\linewidth}
  \centering
  \includegraphics[width=\linewidth]{Fig/lr_legend.PNG} 
\end{subfigure}

\begin{subfigure}{0.24\linewidth}
  \centering
  \includegraphics[width=\linewidth]{Fig/tinymin.png}
  \caption{Tiny ImageNet $\eta_{\min}$}
\end{subfigure}
\begin{subfigure}{0.24\linewidth}
  \centering
  \includegraphics[width=\linewidth]{Fig/tinymax.png}
  \caption{Tiny ImageNet $\eta_{\max}$}
\end{subfigure}
\begin{subfigure}{0.24\linewidth}
  \centering
  \includegraphics[width=\linewidth]{Fig/cifarmin.png}
  \caption{CIFAR-10 $\eta_{\min}$}
\end{subfigure}
\begin{subfigure}{0.24\linewidth}
  \centering
  \includegraphics[width=\linewidth]{Fig/cifarmax.png}
  \caption{CIFAR-10 $\eta_{\max}$}
\end{subfigure}
\begin{subfigure}{0.24\linewidth}
  \centering
  \includegraphics[width=\linewidth]{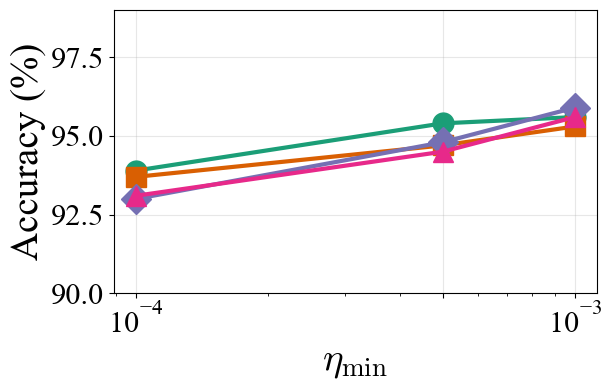}
  \caption{LAMA $\eta_{\min}$}
\end{subfigure}
\begin{subfigure}{0.24\linewidth}
  \centering
  \includegraphics[width=\linewidth]{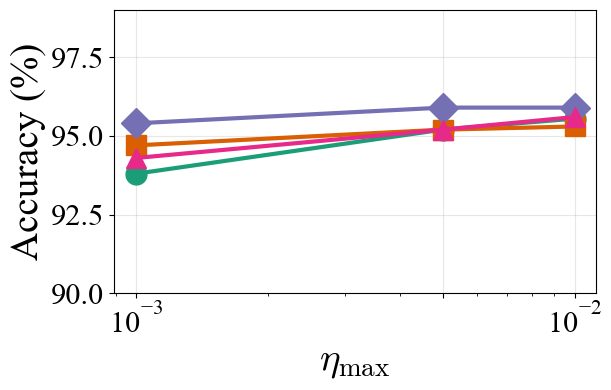}
  \caption{LAMA $\eta_{\max}$}
\end{subfigure}
\begin{subfigure}{0.24\linewidth}
  \centering
  \includegraphics[width=\linewidth]{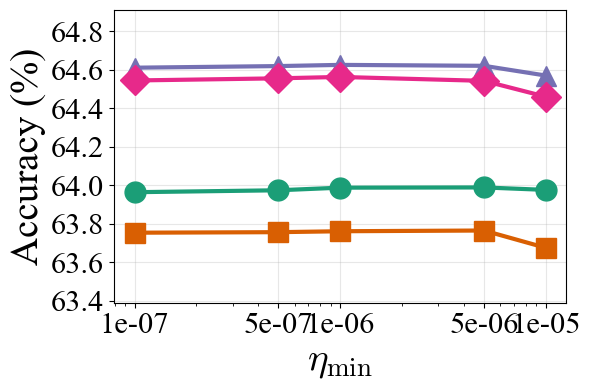}
  \caption{CIFAR-10-C $\eta_{\min}$}
\end{subfigure}
\begin{subfigure}{0.24\linewidth}
  \centering
  \includegraphics[width=\linewidth]{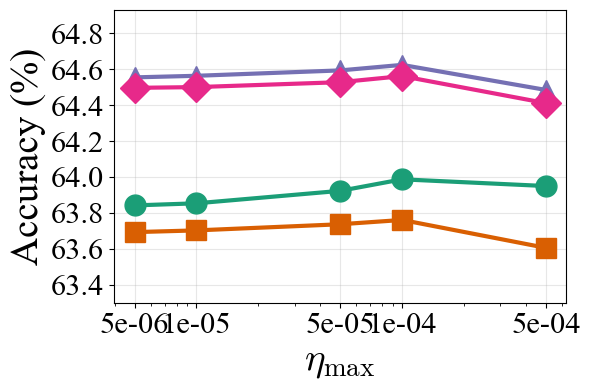}
  \caption{CIFAR-10-C $\eta_{\max}$}
\end{subfigure}
\begin{subfigure}{0.24\linewidth}
  \centering
  \includegraphics[width=\linewidth]{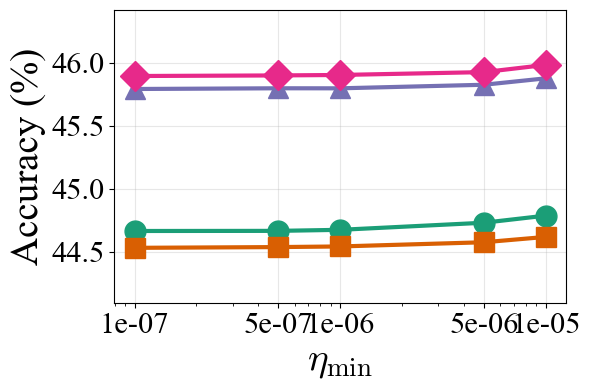}
  \caption{CIFAR-100-C $\eta_{\min}$}
\end{subfigure}
\begin{subfigure}{0.24\linewidth}
  \centering
  \includegraphics[width=\linewidth]{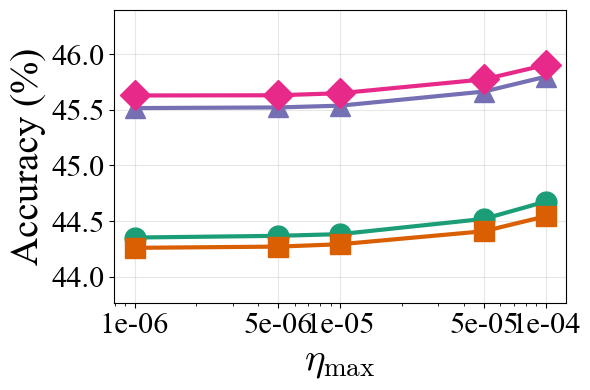}
  \caption{CIFAR-100-C $\eta_{\max}$}
\end{subfigure}


\caption{Impact of learning rate bounds on Fed-ADE under label shift scenarios.}
\label{fig:lr2}
\end{figure*}

Figure~\ref{fig:lr2} examines the sensitivity of Fed-ADE to the choice of learning rate bounds under all scenarios and benchmarks. 
For image benchmarks, we vary $\eta_{\min}$ while fixing $\eta_{\max}=10^{-4}$, and vary $\eta_{\max}$ while fixing $\eta_{\min}=5\times10^{-6}$. For text benchmark LAMA, we vary $\eta_{\min}$ while fixing $\eta_{\max}=10^{-2}$, and vary $\eta_{\max}$ while fixing $\eta_{\min}=10^{-4}$.
Across all cases, Fed-ADE consistently maintains high accuracy within a broad range of hyperparameter choices. 
While extreme values lead to slight performance degradation, the overall trend shows that our dynamics-driven adaptation is far less sensitive to the exact learning rate bounds compared to fixed-rate baselines. 
This robustness highlights that Fed-ADE reduces the need for extensive hyperparameter tuning, an essential property in federated and online learning where search budgets are limited.

\newpage
\subsection{Impact of Distribution of Pre-training Data}
\label{appendix:simulation_dist}

In practical federated deployments, the exact distribution of the pre-training data on the server is often unknown to clients. To examine the robustness of Fed-ADE under such realistic conditions, we evaluate post-adaptation performance on image benchmarks when the pre-training data follows two non-uniform distributions: \emph{Gaussian} and \emph{Exponential Decay}. As shown in Table~\ref{sim:performance_compare_g}, and Table~\ref{sim:performance_compare_e}, Fed-ADE maintains consistently strong performance across all distribution settings and shift types, with only a minor variation compared to the uniform case (cf. Table~\ref{sim:performance_compare}). These results demonstrate that our dynamics-based adaptation effectively generalizes even when the pre-training distribution differs from the assumed prior, validating its robustness under practical deployment scenarios.

\begin{table*}[t]
\centering
\caption{Performance comparison of different online distribution shift adaptation methods across datasets and shift types with pretrained model from Gaussian distribution (average accuracy (\%) and average wall time (sec.))}
\scriptsize
\setlength{\tabcolsep}{2pt} 
\resizebox{\textwidth}{!}{%
\begin{tabular}{@{}c c | c c c c| c c c c c c}
\toprule
\multicolumn{2}{c|}{\textbf{}} & \multicolumn{4}{c|}{\textbf{Localized Learning}} & \multicolumn{6}{c}{\textbf{Federated Learning}} \\
\cmidrule(r){3-6} \cmidrule(l){7-12}
\textbf{Dataset} & \textbf{Shift} & FTH & ATLAS & UNIDA & UDA & Fed-POE & FedCCFA &FixLR(Low) & FixLR(Mid) & FixLR(High) & \textbf{Fed-ADE} \\
\midrule
\rowcolor{gray!15}  
\multicolumn{12}{c}{\textbf{(i) Label Shift Scenarios}} \\\midrule 

\multirow{4}{*}{\shortstack{Tiny\\ImageNet}}
& Lin. &  74.1{\tiny$\pm$2.3} & 70.7{\tiny$\pm$2.1}& 75.1{\tiny$\pm$2.1} & 63.1{\tiny$\pm$2.3} & 81.5{\tiny$\pm$1.2} & 81.7{\tiny$\pm$1.8} &81.5{\tiny$\pm$1.1} & 81.2{\tiny$\pm$1.3} & 78.1{\tiny$\pm$1.5} &\textbf{84.5{\tiny$\pm$1.1}}  \\
& Sin. &  73.7{\tiny$\pm$2.4} & 71.1{\tiny$\pm$1.9} & 74.8{\tiny$\pm$2.5} & 61.3{\tiny$\pm$2.1} & 82.1{\tiny$\pm$1.2} & 80.9{\tiny$\pm$1.1} &81.9{\tiny$\pm$1.7} & 81.0{\tiny$\pm$2.1}& 77.5{\tiny$\pm$2.1} & \textbf{84.7{\tiny$\pm$0.7}} \\
& Squ. & 72.6{\tiny$\pm$2.4} & 70.9{\tiny$\pm$1.8} &  74.6{\tiny$\pm$2.2} & 61.7{\tiny$\pm$2.1} & 81.5{\tiny$\pm$1.3} & 78.8{\tiny$\pm$1.9} & 82.4{\tiny$\pm$1.2} & 80.9{\tiny$\pm$1.6} & 77.7{\tiny$\pm$1.3} &\textbf{83.9{\tiny$\pm$0.9}} \\
& Ber. & 72.1{\tiny$\pm$2.3} & 71.0{\tiny$\pm$2.0} & 74.7{\tiny$\pm$2.1} & 61.1{\tiny$\pm$1.7} & 80.9{\tiny$\pm$1.6} & 79.1{\tiny$\pm$1.1} & 82.1{\tiny$\pm$0.9} & 80.8{\tiny$\pm$1.7} & 77.8{\tiny$\pm$1.5} &\textbf{84.2{\tiny$\pm$0.4}} \\\midrule

\multirow{4}{*}{CIFAR-10}
& Lin. & 28.4{\tiny$\pm$1.7} & 22.6{\tiny$\pm$3.3} &17.2{\tiny$\pm$4.1} & 27.4{\tiny$\pm$4.2} & 66.1{\tiny$\pm$1.2} & 53.6{\tiny$\pm$2.1} &64.8{\tiny$\pm$2.2} & 65.7{\tiny$\pm$2.1} & 51.3{\tiny$\pm$3.1}& \textbf{68.0{\tiny$\pm$1.7}} \\
& Sin. & 28.9{\tiny$\pm$2.1} & 23.1{\tiny$\pm$3.1} &18.3{\tiny$\pm$4.5} & 27.1{\tiny$\pm$3.7} & 66.7{\tiny$\pm$1.2} & 53.3{\tiny$\pm$2.1} & 64.7{\tiny$\pm$1.9} & 65.1{\tiny$\pm$1.8} & 51.6{\tiny$\pm$4.1} & \textbf{67.7{\tiny$\pm$1.7}}\\
& Squ. & 25.9{\tiny$\pm$1.7} & 21.6{\tiny$\pm$2.6} &15.8{\tiny$\pm$3.7} & 25.1{\tiny$\pm$2.2} & 67.1{\tiny$\pm$1.5} &54.3{\tiny$\pm$1.3} & 64.8{\tiny$\pm$1.2} & 65.7{\tiny$\pm$1.6} & 52.3{\tiny$\pm$3.5} & \textbf{66.4{\tiny$\pm$1.6}}\\
& Ber. & 26.1{\tiny$\pm$2.1} & 21.7{\tiny$\pm$2.9} &16.1{\tiny$\pm$4.1} & 24.5{\tiny$\pm$2.5} &67.0{\tiny$\pm$1.3} & 54.1{\tiny$\pm$1.3} & 64.3{\tiny$\pm$1.7} & 65.4{\tiny$\pm$1.4} & 52.6{\tiny$\pm$2.8} & \textbf{67.1{\tiny$\pm$1.3}} \\\midrule

\rowcolor{gray!15}  
\multicolumn{12}{c}{\textbf{(ii) Covariate Shift Scenarios}} \\\midrule 
\multirow{4}{*}{\shortstack{CIFAR-10\\CIFAR-10-C}}
& Lin. &  15.4{\tiny$\pm$0.4} & 10.1{\tiny$\pm$0.8} & 39.1{\tiny$\pm$0.5} & 37.8{\tiny$\pm$1.1} & 41.6{\tiny$\pm$0.4} &35.2{\tiny$\pm$0.5}&58.0{\tiny$\pm$0.6} & 58.1{\tiny$\pm$0.6} & 36.7{\tiny$\pm$2.4} & \textbf{59.5{\tiny$\pm$0.5}} \\
& Sin. &  15.2{\tiny$\pm$0.5} & 11.9{\tiny$\pm$0.7} & 39.2{\tiny$\pm$0.4} & 37.2{\tiny$\pm$0.9} & 41.7{\tiny$\pm$0.6} &33.1{\tiny$\pm$0.4} &57.1{\tiny$\pm$0.5} & 57.8{\tiny$\pm$0.7} & 35.3{\tiny$\pm$1.7} & \textbf{59.3{\tiny$\pm$0.5}} \\
& Squ. &  13.5{\tiny$\pm$0.5} & 13.1{\tiny$\pm$0.6} & 38.3{\tiny$\pm$0.4} & 36.8{\tiny$\pm$1.2} & 43.1{\tiny$\pm$0.5}  &32.1{\tiny$\pm$0.7} &56.9{\tiny$\pm$0.9} & 57.3{\tiny$\pm$1.1} & 35.6{\tiny$\pm$2.1} &\textbf{60.4{\tiny$\pm$0.5}} \\
& Ber. &  14.8{\tiny$\pm$0.5} & 12.2{\tiny$\pm$1.1} & 38.7{\tiny$\pm$0.4} & 36.6{\tiny$\pm$1.1} & 43.2{\tiny$\pm$0.5} &33.0{\tiny$\pm$0.5}&56.7{\tiny$\pm$0.7} & 67.1{\tiny$\pm$1.0} & 46.1{\tiny$\pm$2.7} & \textbf{60.3{\tiny$\pm$0.4}}\\\midrule

\multirow{4}{*}{\shortstack{CIFAR-100\\CIFAR-100-C}}
& Lin. &  4.2{\tiny$\pm$1.7} & 2.4{\tiny$\pm$0.2} & 29.1{\tiny$\pm$0.1} & 31.3{\tiny$\pm$0.3} & 20.2{\tiny$\pm$0.4} & 20.7{\tiny$\pm$0.6} &42.0{\tiny$\pm$0.4} & 40.8{\tiny$\pm$0.2} &37.1{\tiny$\pm$1.1} & \textbf{43.1{\tiny$\pm$0.5}} \\
& Sin. & 5.6 {\tiny$\pm$2.2} & 2.5{\tiny$\pm$0.1} & 29.5{\tiny$\pm$0.3} & 30.7{\tiny$\pm$0.8} & 20.3{\tiny$\pm$0.5} & 20.7{\tiny$\pm$0.6} & 41.9{\tiny$\pm$0.4} &40.9{\tiny$\pm$0.3} &36.6{\tiny$\pm$1.3} & \textbf{44.4{\tiny$\pm$0.4}}  \\
& Squ. & 5.7{\tiny$\pm$2.1}  & 3.2{\tiny$\pm$0.3} & 27.1{\tiny$\pm$0.4}  & 29.1{\tiny$\pm$0.7}  & 22.9{\tiny$\pm$0.4} &  18.6{\tiny$\pm$0.8} & 41.6{\tiny$\pm$0.3}& 41.2{\tiny$\pm$0.5}&37.1{\tiny$\pm$1.5} & \textbf{45.1{\tiny$\pm$0.5}}  \\
& Ber. & 6.1{\tiny$\pm$1.8}  & 3.1{\tiny$\pm$0.4} & 27.8{\tiny$\pm$0.4}  & 29.6{\tiny$\pm$0.8}  & 21.8{\tiny$\pm$0.8} &  19.3{\tiny$\pm$0.8} & 40.3{\tiny$\pm$0.5}&41.4{\tiny$\pm$0.5} &34.1{\tiny$\pm$1.4} &\textbf{44.8{\tiny$\pm$0.6}}  \\




\bottomrule
\end{tabular}}
\vspace{-0.3cm}
\label{sim:performance_compare_g}
\end{table*}

\begin{table*}[t]
\centering
\caption{Performance comparison of different online distribution shift adaptation methods across datasets and shift types  with pretrained model from exponential decay distribution(average accuracy (\%) and average wall time (sec.))}
\scriptsize
\setlength{\tabcolsep}{2pt} 
\resizebox{\textwidth}{!}{%
\begin{tabular}{@{}c c | c c c c| c c c c c c}
\toprule
\multicolumn{2}{c|}{\textbf{}} & \multicolumn{4}{c|}{\textbf{Localized Learning}} & \multicolumn{6}{c}{\textbf{Federated Learning}} \\
\cmidrule(r){3-6} \cmidrule(l){7-12}
\textbf{Dataset} & \textbf{Shift} & FTH & ATLAS & UNIDA & UDA & Fed-POE & FedCCFA &FixLR(Low) & FixLR(Mid) & FixLR(High) & \textbf{Fed-ADE} \\
\midrule
\rowcolor{gray!15}  
\multicolumn{12}{c}{\textbf{(i) Label Shift Scenarios}} \\\midrule 

\multirow{4}{*}{\shortstack{Tiny\\ImageNet}}
& Lin. &  78.2{\tiny$\pm$1.0} & 70.2{\tiny$\pm$2.1} & 77.3{\tiny$\pm$0.8} & 69.0{\tiny$\pm$0.6} & 79.9{\tiny$\pm$0.3} & 82.1{\tiny$\pm$0.4} &81.8{\tiny$\pm$0.5} & 82.5{\tiny$\pm$0.3} & 81.4{\tiny$\pm$0.2} &\textbf{83.1{\tiny$\pm$0.4}}  \\
& Sin. &  77.9{\tiny$\pm$0.8} & 71.4{\tiny$\pm$1.6} & 77.1{\tiny$\pm$0.6} & 69.4{\tiny$\pm$0.5} & 80.4{\tiny$\pm$0.6} & 81.5{\tiny$\pm$0.4} &82.1{\tiny$\pm$0.2} & 82.7{\tiny$\pm$0.5}& 81.3{\tiny$\pm$0.3} & \textbf{82.7{\tiny$\pm$0.5}} \\
& Squ. & 77.2{\tiny$\pm$0.8} &  71.1{\tiny$\pm$1.8} &  76.9{\tiny$\pm$0.6} & 69.4{\tiny$\pm$0.5} & 80.9{\tiny$\pm$0.5} & 81.2{\tiny$\pm$0.3} & 82.3{\tiny$\pm$0.4} & 82.1{\tiny$\pm$0.4} & 80.1{\tiny$\pm$0.3} &\textbf{82.6{\tiny$\pm$0.3}} \\
& Ber. & 78.2{\tiny$\pm$1.1} &  70.7{\tiny$\pm$1.6} &  76.6{\tiny$\pm$0.7} & 69.7{\tiny$\pm$0.4} & 80.4{\tiny$\pm$0.5} & 80.4{\tiny$\pm$0.3} & 81.1{\tiny$\pm$0.4} & 82.1{\tiny$\pm$0.2} & 80.4{\tiny$\pm$0.3} &\textbf{81.8{\tiny$\pm$0.4}} \\\midrule

\multirow{4}{*}{CIFAR-10}
& Lin. & 27.7{\tiny$\pm$1.4} & 29.0{\tiny$\pm$2.0} &20.2{\tiny$\pm$1.3} & 27.6{\tiny$\pm$1.8} & 65.1{\tiny$\pm$1.0} & 63.8{\tiny$\pm$0.9} &64.1{\tiny$\pm$1.4} & 64.8{\tiny$\pm$1.3} & 60.1{\tiny$\pm$1.3}& \textbf{66.2{\tiny$\pm$1.7}} \\
& Sin. &28.4{\tiny$\pm$1.2} & 28.3{\tiny$\pm$2.0} &20.3{\tiny$\pm$1.1} & 28.1{\tiny$\pm$1.6} & 64.7{\tiny$\pm$0.6} & 63.1{\tiny$\pm$1.3} & 63.8{\tiny$\pm$1.2} & 64.3{\tiny$\pm$1.4} & 60.3{\tiny$\pm$0.8} & \textbf{65.9{\tiny$\pm$1.7}}\\
& Squ. & 27.1{\tiny$\pm$1.2} & 28.5{\tiny$\pm$1.8} &21.1{\tiny$\pm$0.9} & 24.8{\tiny$\pm$1.6} & 64.6{\tiny$\pm$0.9} &62.8{\tiny$\pm$1.1} & 64.5{\tiny$\pm$1.1} & 64.1{\tiny$\pm$1.6} & 61.1{\tiny$\pm$0.9} & \textbf{66.4{\tiny$\pm$1.5}}\\
& Ber. & 26.6{\tiny$\pm$1.5} & 27.1{\tiny$\pm$1.9} &20.7{\tiny$\pm$1.1} & 25.1{\tiny$\pm$2.1} &64.3{\tiny$\pm$1.1} & 62.3{\tiny$\pm$0.7} & 64.7{\tiny$\pm$1.4} & 63.8{\tiny$\pm$1.5} & 60.8{\tiny$\pm$1.2} & \textbf{67.1{\tiny$\pm$1.3}} \\\midrule

\rowcolor{gray!15}  
\multicolumn{12}{c}{\textbf{(ii) Covariate Shift Scenarios}} \\\midrule 
\multirow{4}{*}{\shortstack{CIFAR-10\\CIFAR-10-C}}
& Lin. &  23.7{\tiny$\pm$0.2} & 13.9{\tiny$\pm$0.2} & 43.3{\tiny$\pm$0.6} & 45.7{\tiny$\pm$0.8} & 44.5{\tiny$\pm$0.8} & 43.1{\tiny$\pm$0.5} &63.4{\tiny$\pm$0.2} & 63.9{\tiny$\pm$0.3} & 40.6{\tiny$\pm$2.1} & \textbf{58.3{\tiny$\pm$0.5}} \\
& Sin. &  22.9{\tiny$\pm$0.2} & 13.2{\tiny$\pm$0.2} & 44.8{\tiny$\pm$0.1} & 43.5{\tiny$\pm$0.3} & 44.7{\tiny$\pm$1.2} &43.3{\tiny$\pm$0.4} &63.5{\tiny$\pm$0.3} & 63.7{\tiny$\pm$0.3} & 38.1{\tiny$\pm$3.1} & \textbf{58.2{\tiny$\pm$0.5}} \\
& Squ. & 23.8{\tiny$\pm$0.1} & 14.1{\tiny$\pm$0.4} & 42.2{\tiny$\pm$0.1} & 43.3{\tiny$\pm$0.2} & 48.5{\tiny$\pm$1.2} & 41.6{\tiny$\pm$0.7} & 62.7{\tiny$\pm$2.1} & 64.5{\tiny$\pm$2.1} & 39.6{\tiny$\pm$2.8} &\textbf{59.0{\tiny$\pm$0.5}} \\
& Ber. &  23.6{\tiny$\pm$0.2} & 14.2{\tiny$\pm$0.3} & 42.7{\tiny$\pm$0.2} & 42.3{\tiny$\pm$0.1} & 48.7{\tiny$\pm$0.9} & 42.0{\tiny$\pm$0.5} &64.1{\tiny$\pm$2.1} & 64.4{\tiny$\pm$2.1} & 40.8{\tiny$\pm$2.2} & \textbf{59.0{\tiny$\pm$0.4}}\\\midrule

\multirow{4}{*}{\shortstack{CIFAR-100\\CIFAR-100-C}}
& Lin. & 5.1{\tiny$\pm$1.8} & 6.4{\tiny$\pm$0.4} & 24.4{\tiny$\pm$0.5} & 27.6{\tiny$\pm$0.5} & 19.1{\tiny$\pm$2.1} & 17.9{\tiny$\pm$1.1}   &40.8{\tiny$\pm$0.5} & 39.1{\tiny$\pm$0.6} &37.7{\tiny$\pm$1.2} & \textbf{42.7{\tiny$\pm$0.5}} \\
& Sin. & 5.3 {\tiny$\pm$1.7} & 5.1{\tiny$\pm$0.4} & 24.3{\tiny$\pm$0.6} & 28.1{\tiny$\pm$0.5} & 20.1{\tiny$\pm$1.9} & 18.3{\tiny$\pm$1.1}   & 40.7{\tiny$\pm$0.6} &38.6{\tiny$\pm$0.7} &37.6{\tiny$\pm$1.2} & \textbf{42.4{\tiny$\pm$0.6}}  \\
& Squ. & 4.8{\tiny$\pm$1.5}  & 4.7{\tiny$\pm$0.5} & 24.1{\tiny$\pm$0.4}  & 28.3{\tiny$\pm$0.7}  & 19.8{\tiny$\pm$2.1} &  18.1{\tiny$\pm$0.9} & 39.4{\tiny$\pm$0.6}& 38.2{\tiny$\pm$0.7}&38.1{\tiny$\pm$1.1} & \textbf{42.1{\tiny$\pm$0.4}}  \\
& Ber. & 4.3{\tiny$\pm$1.3}  & 5.1{\tiny$\pm$0.3} & 23.5{\tiny$\pm$0.4}  & 27.8{\tiny$\pm$0.6}  & 19.7{\tiny$\pm$2.3} &  16.9{\tiny$\pm$1.0} & 40.1{\tiny$\pm$0.4}&37.7{\tiny$\pm$0.6} &37.8{\tiny$\pm$1.2} &\textbf{42.0{\tiny$\pm$0.5}}  \\



\bottomrule
\end{tabular}}
\vspace{-0.3cm}
\label{sim:performance_compare_e}
\end{table*}

\subsection{Why Cosine Similarity? Alternatives and Ablations}
\label{app:cosin}

In this paper, our Fed-ADE must be (i) online \& per-client (no extra communication), (ii) label-free (only pseudo-labels/representations), (iii) robust to pseudo-label noise and transient shifts, (iv) lightweight in compute and memory, and (v) bounded and interpretable so that we can stably map it to a learning rate in $[\eta_{\min},\eta_{\max}]$ across heterogeneous clients.

\paragraph{Cosine similarity meets these requirements}

\begin{itemize}[leftmargin=1.2em,itemsep=0.2ex]
\item \textbf{Bounded}: $S_c^t$ is bounded in [0,1], so it leads to a stable mapping to learning across clients and timestep
\item \textbf{Efficient}: Linear-time, no optimization/probabilistic inference
\item \textbf{Robust}: Scale-invariant and less sensitive to sparsity/outliers in pseudo-label histograms; temperature rescaling or mild miscalibration does not change directions.

\end{itemize}

\paragraph{Alternatives considered}
We implemented three commonly suggested alternatives and found consistent drawbacks in our setting.

\noindent\textbf{(1) KL divergence on batch histograms}:  
\emph{Pros:} classical information-theoretic distance.  
\emph{Cons:} asymmetric and \emph{unbounded}; becomes unstable/infinite under zeros/sparsity—common with noisy pseudo-labels on small batches. Requires smoothing/regularization and hyperparameters (e.g., $\epsilon$-floor), which are highly dataset/client dependent and harm cross-client consistency. Mapping an unbounded statistic to a stable learning rate is delicate.

\noindent\textbf{(2) Wasserstein distance}:  
Requires a ground metric on classes or high-dim features; even with Sinkhorn regularization, it needs iterative optimization (per batch pair) and careful tuning. \emph{Pros:} captures geometric discrepancies. \emph{Cons:} higher compute and memory; sensitive to transient fluctuations; scale depends on ground metric $\Rightarrow$ learning rate scaling across clients becomes ill-conditioned. In high-dim feature space, pairwise costs/transport exacerbate overhead.

\noindent\textbf{(3) Bayesian change-point detection}:  
\emph{Pros:} probabilistic treatment of regime switches. \emph{Cons:} assumes i.i.d.\ within segments and requires prior modeling; posterior updates/inference add computation; fragile to batch-level noise (false alarms/delays). Typically designed for detection (binary), not for per-step \emph{continuous} shift magnitudes needed to drive learning rate.

\paragraph{Performance comparison of cosine similarity vs. alternatives}
We compared cosine similarity with KL, Wasserstein, and Bayesian CPD under four dynamic shift patterns (Lin./Sin./Squ./Ber.) on CIFAR-10. Cosine consistently achieved the best accuracy:

\begin{center}
\begin{tabular}{lcccc}
\toprule
\textbf{Measure} & \textbf{Lin.} & \textbf{Sin.} & \textbf{Squ.} & \textbf{Ber.} \\
\midrule
Fed-ADE (Cosine similarity) & \textbf{73.8 $\pm$ 0.6} & \textbf{73.6 $\pm$ 0.5} & \textbf{72.2$\pm$ 1.6} & \textbf{72.9 $\pm$ 2.2} \\
KL divergence & 63.2 $\pm$ 3.8 & 62.4 $\pm$ 2.3 & 61.1 $\pm$ 3.4 & 61.4 $\pm$ 3.6 \\
Wasserstein & 70.8 $\pm$ 0.6 & 69.9 $\pm$ 1.1 & 68.8 $\pm$ 1.3 & 66.5 $\pm$ 2.1 \\
Bayesian CPD & 71.5 $\pm$ 0.5 & 70.8 $\pm$ 0.8 & 69.7 $\pm$ 2.4 & 69.0 $\pm$ 1.1\\
\bottomrule
\end{tabular}
\end{center}

Gains of the cosine similarity method are largest under label noise and transient shifts, where bounded, direction-only comparison stabilizes learning rate scaling across clients.

\end{document}